\documentclass{article}

\usepackage[numbers]{natbib}
\usepackage{cite}
\usepackage{wrapfig}
\usepackage{mathrsfs}
\usepackage{times}
\usepackage{enumerate}
\usepackage{color}
\usepackage{graphicx,epsfig,subfigure}
\usepackage{algorithm,algorithmic}
\usepackage{amsmath,amssymb,xspace,amsthm}
\usepackage{url}
\usepackage{subfigure}
\usepackage{xspace}
\usepackage{mathrsfs}
\usepackage{bbm}
\usepackage{upgreek}
\usepackage{bookmark}
\usepackage{hyperref}

\newcommand{\ie}{{{i.e.,}}\xspace}
\newcommand{\eg}{{{\em e.g.,}}\xspace}
\newcommand{\cmt}[1]{}
\newcommand{\ours}{{{our model}}\xspace}
\newcommand{\hadoop}{\textsc{Hadoop}\xspace}
\newcommand{\mapreduce}{\textsc{MapReduce}\xspace}
\newcommand{\spark}{\textsc{SPARK}\xspace}
\newcommand{\map}{\textsc{Map}\xspace}
\newcommand{\mapper}{\textsc{Mapper}\xspace}
\newcommand{\mappers}{\textsc{Mappers}\xspace}
\newcommand{\reduce}{\textsc{Reduce}\xspace}
\newcommand{\reducer}{\textsc{Reducer}\xspace}
\newcommand{\InfTuckerEx}{{InfTuckerEx}\xspace}
\newcommand{\InfTucker}{{InfTucker}\xspace}

\newcommand{\alanc}[1]{}

\newcommand{\email}[1]{\href{mailto:#1}{#1}}

\newcommand{\expt}[1]{\langle #1 \rangle}
\newcommand{\tr}{{\rm tr}}
\renewcommand{\vec}{{\rm vec}}

\renewcommand{\a}{{\bf a}}
\renewcommand{\b}{{\bf b}}

\renewcommand{\d}{{\rm d}}  % for derivatives

\newcommand{\f}{{\bf f}}

\newcommand{\bi}{{\bf i}}
\newcommand{\bj}{{\bf j}}

\renewcommand{\k}{{\bf k}}
\newcommand{\m}{{\bf m}}

\renewcommand{\u}{{\bf u}}
\renewcommand{\v}{{\bf v}}

\newcommand{\x}{{\bf x}}
\newcommand{\y}{{\bf y}}
\newcommand{\z}{{\bf z}}
\newcommand{\A}{{\bf A}}
\newcommand{\B}{{\bf B}}

\newcommand{\D}{{\bf D}}
\newcommand{\E}{{\bf E}}

\newcommand{\I}{{\bf I}}

\newcommand{\K}{{\bf K}}

\newcommand{\Mcal}{{\mathcal{M}}}
\newcommand{\N}{\mathcal{N}}  % for normal density
\newcommand{\bupeta}{\boldsymbol{\upeta}}

\newcommand{\U}{{\bf U}}
\newcommand{\Ucal}{{\mathcal{U}}}

\newcommand{\Wcal}{{\mathcal{W}}}
\newcommand{\X}{{\bf X}}

\newcommand{\Ycal}{{\mathcal{Y}}}

\newcommand{\blambda}{\boldsymbol{\lambda}}
\newcommand{\bLambda}{\mathbf{\Lambda}}

\newcommand{\btheta}{\boldsymbol{\theta}}

\newcommand{\bSigma}{\boldsymbol{\Sigma}}

\newcommand{\bmu}{\boldsymbol{\mu}}

\newcommand{\0}{{\bf 0}}

\newcommand{\ben}{\begin{enumerate}}
\newcommand{\een}{\end{enumerate}}

\begin{document} 
%\abovedisplayskip=5pt
%\abovedisplayshortskip=0pt
%\belowdisplayskip=5pt
%\belowdisplayshortskip=0pt

\newtheorem{theorem}{Theorem}[section]
\newtheorem{corollary}{Corollary}[theorem]
\newtheorem{lem}[theorem]{Lemma}

\title{Distributed Flexible Nonlinear Tensor Factorization}
\author{
       Shandian Zhe\\
       Purdue University\\
       \email{szhe@purdue.edu}
       \and
       Kai Zhang\\
       Lawrence Berkeley Lab \\
       \email{kzhang980@gmail.com}
       \and
        Pengyuan Wang \\
        Yahoo! Research \\
        \email{pengyuan@yahoo-inc.com}
        \and
        Kuang-chih Lee\\
        Yahoo! Research\\
        \email{kclee@yahoo-inc.com}
        \and
        Zenglin Xu\\
		University of Electronic Science and Technology of China\\
		\email{zlxu@uestc.edu.cn}
		\and
		Yuan Qi\\
		Purdue University\\
		\email{alanqi@cs.purdue.edu}
		\and
		Zoubin Ghahraman\\
		University of Cambridge\\
		\email{zoubin@eng.cam.ac.uk}
}

\maketitle

%\vspace{-0.2in}
\begin{abstract}
%\vspace{-0.12in}
%Tensor factorization is an important approach to multiway data analysis.Compared with popular multilinear methods, nonlinear tensor factorization models are able to capture more complex relationships in data. However, they are computationally expensive and may suffer from the data sparsity. To overcome these limitations,  we propose a new nonlinear factorization model based on Gaussian processes, which can avoid the expensive computation of the Kronecker product and flexibly incorporate meaningful entries for training. To scale up the model to large data, we develop a distributed variational inference algorithm in \mapreduce framework. To this end, we derive a tractable and tight variational evidence lower bound (ELBO) that enables efficient parallel computations and high quality inferences. In addition, we design a key-value-free \map -\reduce scheme that can prevent the costly data shuffling on disk and fully use the memory-cache mechanism in fast \mapreduce systems, such as \spark.
Tensor factorization is a powerful tool to analyse multi-way data. Compared with traditional multi-linear methods, nonlinear tensor factorization models are capable of capturing more complex relationships in the data. However, they are computationally expensive and may suffer severe learning bias in case of extreme data sparsity. To overcome these limitations,  in this paper we propose a distributed, flexible nonlinear tensor factorization model. 
Our model can effectively avoid the expensive computations and structural restrictions of the Kronecker-product in existing TGP formulations, allowing an arbitrary subset of tensorial entries to be selected to contribute to the training. At the same time, we derive a tractable and tight variational evidence lower bound (ELBO) that enables highly decoupled, parallel computations and high-quality inference.
Based on the new bound, we develop a distributed inference algorithm in the \mapreduce framework, which is key-value-free and can fully exploit the memory cache mechanism in fast \mapreduce systems such as \spark. 
Experimental results fully demonstrate the advantages of our method over several state-of-the-art approaches,  in terms of both predictive performance and computational efficiency. Moreover, our approach shows a promising potential in the application of Click-Through-Rate (CTR) prediction for online advertising.

\end{abstract}
 
%\vspace{-0.2in}
\section{Introduction}
%\vspace{-0.1in}
Tensors, or multidimensional arrays, are generalizations of matrices (from binary interactions) to high-order interactions between multiple entities. For example, we can extract a three-mode tensor (\textit{user}, \textit{advertisement}, \textit{context}) from online advertising data. To analyze tensor data, people usually turn to factorization approaches that use a set of latent factors to represent each entity and model how the latent factors interact with each other to generate tensor elements. Classical tensor factorization models include Tucker ~\citep{Tucker66} and CANDECOMP/PARAFAC (CP)~\citep{Harshman70parafac} decompositions, which have been widely used in real-world applications. However, because they all assume a multi-linear interaction between the latent factors, they are unable to capture more complex, nonlinear relationships. Recently, ~\citet{XuYQ12} proposed Infinite Tucker decomposition (InfTucker), which generalizes the Tucker model to infinite feature space using a Tensor-variate Gaussian process (TGP) and thus is powerful to model intricate nonlinear interactions. However, InfTucker and its variants~\citep{zhe2013dintucker, zhe2015scalable} are computationally expensive, because the Kronecker product between the covariances of all the modes requires the TGP to model the entire tensor structure. %What's more, they can suffer from the extreme sparsity of data, \ie the highly imbalanced amount of nonzero entries. As we know,  in real tensor data, most zero elements are meaningless---they are just missing or unobserved, and employing all the zero entries in the model may affect the factorization quality and lead to biased prediction.
In addition, they may suffer from the extreme sparsity of real-world tensor data, \ie when the proportion of the nonzero entries is extremely low. As is often the case, most of the zero elements in real tensors are meaningless: they simply indicate missing or unobserved entries. Incorporating all of them in the training process may affect the factorization quality and lead to biased predictions.

%While many tensor factorization algorithms have been proposed, they rarely address all these challenges.
%Standard multilinear methods, such as Tucker decomposition~\citep{Tucker66} and CANDECOMP/PARAFAC (CP)~\citep{Harshman70parafac} are difficult to model complex and nonlinear interactions between tensor modes. Although nonlinear methods such as Infinite Tucker Decomposition(\InfTucker)~\citep{XuYQ12} and its extensions~\citep{zhe2013dintucker, zhe2015scalable}, can overcome this limit by using tensor-variate Gaussian processes (GP)~\citep{XuYQ12},  they are computationally expensive since the Kronecker product between the covariances of all the modes forces the GP to rigidly model an entire tensor. What's more, they are incapable to model the extreme sparsity of data, \ie the imbalanced amount of nonzero entries. As we know,  in real data, most zero elements are meaningless---they are just missing or unobserved, thus employing all the zero entries in the modelling may affect the factorization quality and lead to biased prediction.

To address these issues, in this paper we propose a distributed, flexible nonlinear tensor factorization model, which has several important advantages. First, it can capture highly nonlinear interactions in the tensor, and is flexible enough to incorporate arbitrary subset of (meaningful) tensorial entries for the training.
This is achieved by placing Gaussian process priors over tensor entries, where the input is constructed by concatenating the latent factors from each mode and the intricate relationships are captured by using the kernel function. By using such a construction, the covariance function is then free of the Kronecker-product structure, and as a result users can freely choose any subset of tensor elements for the training process and incorporate prior domain knowledge. For example, one can choose a combination of balanced zero and nonzero elements to overcome the learning bias. %to scale up the model to large data, we develop a distributed variational inference algorithm in \mapreduce framework: First, we derive a tight variational evidence lower bound (ELBO) using functional derivatives and convex conjugates. The tight ELBO subsumes the optimal variational posteriors, thus evades the low-efficient, alternative EM updates and enables efficient, parallel computations as well as improved inference quality. Then we develop a distributed gradient-based approach to optimize the tight ELBO. For binary data, we in addition combine with efficient fixed point iterations. %Our method for binary data is suitable for GP classification problem.
%Finally, to speed up training, we further devise a key-value-free \map-\reduce trick to prevent the costly data shuffling on disk and to exploit the memory cache mechanism in fast \mapreduce systems such as \spark.
 Second, the tight variational evidence lower bound (ELBO) we derived using functional derivatives and convex conjugates subsumes optimal variational posteriors, thus evades inefficient, sequential E-M updates and enables highly efficient, parallel computations as well as improved inference quality. Moreover, the new bound allows us to develop a distributed, gradient-based optimization algorithm. %For binary data we in addition combine with efficient fixed point iterations.
 Finally, we develop a simple yet very efficient procedure to avoid the data shuffling operation, a major performance bottleneck in the (key-value) sorting procedure in \mapreduce.
 That is, rather than sending out key-value pairs, each mapper simply calculates and sends a global gradient vector without keys. This key-value-free procedure is general and can effectively prevent massive disk IOs and fully exploit the memory cache mechanism in fast \mapreduce systems, such as \spark.

%In experiments, we first examine our model on small real datasets where most of the tensor factorization methods are feasible. Our model obtains higher prediction accuracy than \InfTucker and the other methods. Then
Evaluation using small real-world tensor data have fully demonstrated the superior prediction accuracy of our model in comparison with \InfTucker and other state-of-the-art.
On large tensors with millions of nonzero elements, our approach is significantly better than, or at least as good as two popular large-scale nonlinear factorization methods based on TGP: one uses hierarchical modeling to perform distributed infinite Tucker decomposition~\citep{zhe2013dintucker}; the other further enhances InfTucker by using Dirichlet process mixture prior over the latent factors and employs an online learning scheme~\citep{zhe2015scalable}. Our method also outperforms GigaTensor~\citep{kang2012gigatensor}, a typical large-scale CP factorization algorithm, by a large margin. In addition, our method achieves faster training speed and enjoys almost linear scalability on the number of computational nodes. We apply our model to CTR prediction for online advertising and achieves a significant, $20\%$ improvement over the popular logistic regression and linear SVM approaches.

 %Then the examination on large tensor data with millions of nonzero elements has shown that our method can obtain much better predictive performance than GigaTensor, the state-of-the-art large-scale multilinear factorization algorithm, and significantly better than or comparable with scalable nonlinear methods based on tensor-variate GP.  Moreover, our method achieves faster training speed and enjoys almost linear scalability on the number of computational nodes. Finally, our model is applied to CTR prediction for online advertising and achieves a $20\%$ improvement over the broadly used logistic regression and linear SVM.

%For evaluation, the proposed model is first examined on four small real world datasets where most of  tensor factorization methods are feasible. It shows that our method obtains higher prediction accuracy than \InfTucker and other multilinear methods.  Then the examination on large tensor data with millions of nonzero elements has shown that our method can obtain much better predictive performance than GigaTensor, the state-of-the-art large-scale multilinear factorization algorithm, and significantly better than or comparable with scalable nonlinear methods based on tensor-variate GP.  Moreover, our method achieves faster training speed and enjoys almost linear scalability on the number of computational nodes. Finally, our model is applied to CTR prediction for online advertising and achieves a $20\%$ improvement over the broadly used logistic regression and linear SVM.

%\vspace{-0.1in}
\section{Background}
%\vspace{-0.1in}
%\subsection{Notation and Preliminaries}
We first introduce the background knowledge. For convenience, we will use the same notations in \citep{XuYQ12}. Specifically, we denote a $K$-mode tensor by $\Mcal \in \mathbb{R}^{d_1 \times \ldots \times d_K}$, where the $k$-th mode is of dimension $d_k$. The tensor entry at location $\bi$ ($\bi=(i_1,\ldots, i_K)$) is denoted by $m_{\mathbf{i}}$.  %The tensor $\Mcal$ can be flatten into a  vector, denoted by $\vec(\Mcal)$, where  the entry $\mathbf{i}=(i_1, \ldots, i_K)$ in $\Mcal$ is mapped to the entry at position $j=i_K + \sum_{k=1}^{K-1}(i_k - 1)\prod_{t=k+1}^K m_t$ in $\vec(\Mcal)$.
To introduce Tucker decomposition, we need to generalize matrix-matrix products to tensor-matrix products. Specifically, a tensor $\mathcal{W} \in \mathbb{R}^{r_1 \times \ldots \times r_K}$ can multiply with a matrix $\U \in \mathbb{R}^{s \times t}$ at mode $k$ when its dimension at mode-$k$ is consistent with the number of columns in $\U$, \ie $r_k = t$. The product is a new tensor, with size $r_1 \times \ldots \times r_{k-1} \times s \times r_{k+1} \times \ldots \times r_K$. Each element is calculated by $(\Wcal \times_k \U)_{i_1\ldots i_{k-1} j i_{k+1}\ldots i_K} = \sum_{i_k=1}^{r_k} w_{i_1\ldots i_K}u_{ji_k}.$

The Tucker decomposition model uses a latent factor matrix $\U_k \in \mathbb{R}^{d_k \times r_k}$ in each mode $k$ and a core tensor $\mathcal{W} \in \mathbb{R}^{r_1 \times \ldots \times r_K}$ and assumes the whole tensor $\Mcal$ is generated by $\Mcal = \Wcal \times_1 \U^{(1)} \times_2 \ldots \times_K \U^{(K)}$.
%as follows,
%\begin{align}
%\Mcal &= \Wcal \times_1 \U^{(1)} \times_2 \ldots \times_K \U^{(K)}. %\label{eq:tucker}%\raisetag{.05in}
%\end{align}
Note that this is a multilinear function of $\Wcal$ and $\{\U_1, \ldots, \U_K\}$. It can be further simplified by restricting $r_1 = r_2 = \ldots = r_K$ and the off-diagonal elements of $\Wcal$ to be $0$. In this case, the Tucker model becomes CANDECOMP/PARAFAC (CP).

%\subsection{Nonlinear Tensor Factorization}
The infinite Tucker decomposition (\InfTucker) generalizes the Tucker model to infinite feature space via a tensor-variate Gaussian process (TGP)~\citep{XuYQ12}. Specifically, in a probabilistic framework, we assign a standard normal prior over each element of the core tensor $\Wcal$, and then marginalize out $\Wcal$ to obtain the probability of the tensor given the latent factors:
\begin{align}
p(\Mcal|\U^{(1)},\ldots,\U^{(K)}) =\N(\vec(\Mcal);\0, \Sigma^{(1)}\otimes \ldots \otimes \Sigma^{(K)})\label{eq:tensor-gp}
\end{align}
where $\vec{(\Mcal)}$ is the vectorized whole tensor,  $\Sigma^{(k)} = \U^{(k)}{\U^{(k)}}^\top$ and $\otimes$ is the Kronecker-product. Next, we apply the kernel trick to model nonlinear interactions between the latent factors:  Each row $\u_t^k$ of the latent factors $\U^{(k)}$ is replaced by a nonlinear feature transformation $\phi(\u_t^k)$ and thus an equivalent nonlinear covariance matrix $\Sigma^{(k)} = k(\U^{(k)},\U^{(k)})$ is used to replace $\U^{(k)}{\U^{(k)}}^\top$, where $k(\cdot,\cdot)$ is the covariance function. After the nonlinear feature mapping, the original Tucker decomposition is performed in an (unknown) infinite feature space. Further, since the covariance of $\vec(\Mcal)$ is a function of the latent factors $\Ucal = \{\U^{(1)}, \ldots, \U^{(K)}\}$, Equation \eqref{eq:tensor-gp} actually defines a Gaussian process (GP) on tensors, namely tensor-variate GP (TGP) \citep{XuYQ12}, where the input are based on  $\Ucal$. Finally, we can use different noisy models $p(\Ycal | \Mcal)$ to sample the observed tensor $\Ycal$. For example, we can use Gaussian models and Probit models for continuous and binary observations, respectively.
%Further, the standard Laplace prior is assigned over $\Ucal$ to encourage sparse estimation for easy interpretation. Given $\Mcal$, the observed tensor $\Ycal$ is sampled from a noisy model $p(\Ycal | \Mcal)$. For example, we can use Gaussian models and probit models for continuous and binary observations respectively. Thus the joint probability of \InfTucker is $p(\Ycal, \Mcal, \Ucal) = p(\Ucal)p(\Mcal|\Ucal)p(\Ycal|\Mcal)$.

%\vspace{-0.1in}
\section{Model}
%\vspace{-0.1in}
Despite being able to capture nonlinear interactions, \InfTucker may suffer from the extreme sparsity issue in real-world tensor data sets. %zero and nonzero imbalance in data.
The reason is that its full covariance is a Kronecker-product between the covariances over all the modes---$\{\Sigma^{(1)},\ldots, \Sigma^{(K)}\}$ (see  Equation \eqref{eq:tensor-gp}). Each $\Sigma^{(k)}$ is of size $d_k \times d_k$ and the full covariance is of size $\prod_{k}d_k \times \prod_{k}d_k$. 
Thus TGP is projected onto the entire tensor with respect to the latent factors $\Ucal$, including all zero and nonzero elements, rather than a (meaningful) subset of them.
However, the real-world tensor data are usually extremely sparse, with a huge number of zero entries and a tiny portion of nonzero entries.  On one hand, because most zero entries are meaningless---they are either missing or unobserved, using them can adversely affect the tensor factorization quality and lead to biased predictions; on the other hand, incorporating numerous zero entries into GP models will result in large covariance matrices and high computational costs. Although \citet{zhe2013dintucker,zhe2015scalable} proposed to improve the scalability by modeling subtensors instead, the sampled subtensors can still be very sparse. Even worse, because subtensors are typically restricted to a small dimension due to the efficiency considerations, it is often possible to encounter one that does not contain any nonzero entry. This may further incur numerical instabilities in model estimation.

To address these issues, we propose a flexible Gaussian process tensor factorization model. While inheriting the nonlinear modeling power, our model disposes of the Kronecker-product structure in the full covariance and can therefore select  an arbitrary subset of tensor entries for training. % rather than exclusively using the entire tensor.

Specifically, given a tensor $\Mcal\in  \mathbb{R}^{d_1 \times \ldots \times d_K}$,  for each tensor entry $m_\bi$ ($\bi=(i_1,\ldots, i_K)$), we construct an input $\x_i$ by concatenating the corresponding latent factors from all the modes: $\x_\bi=[\u^{(1)}_{i_1}, \ldots, \u^{(K)}_{i_K}]$, where $\u^{(k)}_{i_k}$ is the $i_k$-th row in the latent factor matrix $\U^{(k)}$ for mode $k$. %Note that $\u^{(k)}_{i_k}$ is a $1\times d_k$ vector and $\x_\bi$ is a $1 \times \sum_{j=1}^K d_j $ vector.  
We assume that  there is an underlying function $f:\mathbb{R}^{\sum_{j=1}^K d_j}\rightarrow \mathbb{R}$ such that $m_\bi = f(\x_\bi) = f([\u^{(1)}_{i_1}, \ldots, \u^{(K)}_{i_K}])$.
%\[
%m_\bi = f(\x_\bi) = f([\u^{(1)}_{i_1}, \ldots, \u^{(K)}_{i_K}]).
%\]
This function is unknown and can be complex and nonlinear. To learn the function, we assign a Gaussian process prior over $f$:  for any set of tensor entries  $S=\{\bi_1, \ldots, \bi_N\}$,  the function values $\f_{S}=\{f(\x_{\bi_1}),\ldots, f(\x_{\bi_N})\}$ are distributed according to a multivariate Gaussian distribution with mean $\0$ and covariance determined by $\X_S=\{\x_{\bi_1}, \ldots, \x_{\bi_N}\}$:
\[
p(\f_{S}|\Ucal) = \N(\f_{S} | \0, k(\X_S, \X_S))
\]
where $k(\cdot, \cdot)$ is a (nonlinear) covariance function. 

Because $k(\x_\bi, \x_\bj)=k([\u^{(1)}_{i_1},\ldots, \u^{(K)}_{i_K}],[\u^{(1)}_{j_1},\ldots, \u^{(K)}_{j_K}])$, there is no Kronecker-product structure constraint and so any subset of tensor entries can be selected for training. To prevent the learning process to be biased toward zero, we can use a set of entries with balanced zeros and nonzeros. Furthermore, useful domain knowledge can also be incorporated to select meaningful entries for training. Note, however, that if we still use all the tensor entries and intensionally impose the Kronecker-product structure in the full covariance, our model is reduced to InfTucker. Therefore, from the modeling perspective, the proposed model is more general.

We further assign a standard normal prior over the latent factors $\Ucal$. Given the selected tensor entries $\m=[m_{\bi_1}, \ldots, m_{\bi_N}]$, the observed entries $\y=[y_{\bi_1}, \ldots, y_{\bi_N}]$ are sampled from a noise model $p(\y | \m)$. In this paper, we deal with both continuous and binary observations. For continuous data, we use the Gaussian model, $p(\y|\m) = \N(\y|\m, \beta^{-1}\I)$ and the joint probability is
\begin{align}
p(\y, \m, \Ucal) &=  \prod\nolimits_{t=1}^K \N(\vec(\U^{(t)})|\0, \I)  \N(\m|\0, k(\X_S, \X_S))  \N(\y|\m, \beta^{-1}\I) \label{eq:joint_prob_reg}
\end{align}
where $S = [\bi_1, \ldots, \bi_N]$. For binary data, we use the Probit model in the following manner. We first introduce augmented variables $\z=[z_1,\ldots, z_N]$ and then decompose the Probit model into
$p(z_j | m_{\bi_j}) = \N(z_j|m_{\bi_j},1)$ and $p(y_{\bi_j}|z_j) = \mathbbm{1}(y_{\bi_j}=0)\mathbbm{1}(z_j \le 0) + \mathbbm{1}(y_{\bi_j}=1)\mathbbm{1}(z_j>0)$ where $\mathbbm{1}(\cdot)$ is the indicator function. Then the joint probability  is
\begin{align}
&p(\y,\z, \m, \Ucal) = \prod\nolimits_{t=1}^K \N(\vec(\U^{(t)})|\0, \I)  \N(\m|\0, k(\X_S, \X_S)) \N(\z|\m, \I) \nonumber \\
\cdot & \prod\nolimits_j \mathbbm{1}(y_{\bi_j}=0)\mathbbm{1}(z_j \le 0) + \mathbbm{1}(y_{\bi_j}=1)\mathbbm{1}(z_j>0). \label{eq:joint_prob_binary}
\end{align}

%\vspace{-0.1in}
\section{Distributed Variational Inference}
%\vspace{-0.1in}
Real-world tensor data often comprise a large number of entries, say, millions of non-zeros and billions of zeros. Even by only using nonzero entries for training, exact inference of the proposed model may still be intractable. This motivates us to develop a distributed variational inference algorithm, presented as follows.
%\vspace{-0.1in}
\subsection{Tractable Variational Evidence Lower Bound}
%\vspace{-0.1in}
Since the GP covariance term --- $k(\X_S, \X_S)$ (see Equations \eqref{eq:joint_prob_reg} and \eqref{eq:joint_prob_binary}) intertwines all the latent factors, exact inference in parallel is  difficult. Therefore, we first derive a tractable variational evidence lower bound (ELBO), following the sparse Gaussian process framework by \citet{titsias2009variational}. The key idea is to introduce a small set of inducing points  $\B=\{\b_1, \ldots, \b_p\}$ and latent targets $\v=\{v_1, \ldots, v_p\}$ ($p \ll N$). Then we augment the original model with a joint multivariate Gaussian distribution of the latent tensor entries  $\m$ and targets $\v$,
\begin{align}
p(\m,\v | \Ucal, \B) = \N(\left[\begin{array}{c}\m \\ \v\end{array}\right] ; \left[\begin{array}{c}\0 \\ \0 \end{array}\right] , \left[\begin{array}{cc}\K_{SS} & \K_{SB} \\ \K_{BS} & \K_{BB}  \end{array}\right]) \nonumber
\end{align}
where $\K_{SS} = k(\X_S, \X_S)$, $\K_{BB} = k(\B, \B)$, $\K_{SB} = k(\X_S, \B)$ and $\K_{BS} = k(\B, \X_S)$. We use Jensen's inequality and conditional Gaussian distributions to construct the ELBO. Using a very similar derivation to \citep{titsias2009variational}, we can obtain a tractable ELBO for our model on continuous data, $\log\big(p(\y, \Ucal | \B)\big) \ge L_1\big(\Ucal, \B, q(\v)\big)$ , where
\begin{align}
L_1\big(\Ucal, \B, q(\v)\big) &= \log(p(\Ucal)) + \int q(\v) \log\frac{ p(\v|\B) }{ q(\v)} \d \v \nonumber \\& +  \sum\nolimits_j \int q(\v) F_\v(y_{\bi_j},\beta) \d \v. \label{eq:elbo_r}
\end{align}
\cmt{
\begin{align}
&\log(p(\y, \Ucal | \B)) \ge \log(p(\Ucal)) + \int q(\v) \log\frac{ p(\v|\B) }{ q(\v)} \d \v \nonumber \\
&+  \sum\nolimits_j \int q(\v) F_\v(y_{\bi_j},\beta) \d \v  \triangleq L_1(\Ucal, \B, q(\v)) , \label{eq:elbo_r}
\end{align}
}
Here $p(\v|\B) = \N(\v|\0, K_{BB})$, $q(\v)$ is the variational posterior for the latent targets $\v$ and
$F_\v(\cdot_j, *) =  \int \log\big(\N(\cdot_j | m_{\bi_j}, *) \big)\N(m_{\bi_j}|\mu_j, \sigma_j^2)\d m_{\bi_j}$,
%\[
%F_\v(\cdot_j, *) =  \int \log\big(\N(\cdot_j | m_{\bi_j}, *) \big)\N(m_{\bi_j}|\mu_j, \sigma_j^2)\d m_{\bi_j}.
%\]
where $\mu_j = k(\x_{\bi_j}, \B)\K_{BB}^{-1}\v$ and $\sigma_j^2 =\bSigma(j,j)  = k(\x_{\bi_j}, \x_{\bi_j}) - k(\x_{\bi_j}, \B)\K_{BB}^{-1}k(\B, \x_{\bi_j}) $. %Note that the term $ F_\v(y_{\bi_j},\beta)$ does not contain $\K_{SS}$ and thus is much easier to calculate.
%Furthermore, the additive form of $L_1$ enables us to distribute the computation across multiple computers.
Note that $L_1$ is decomposed into a summation of terms involving individual tensor entries $\bi_j (1 \le j \le N)$.  The additive form enables us to distribute the computation across multiple computers.
%Note that $F_\v(y_{\bi_j},\beta)$ is the expectation of the log likelihood of a single tensor entry given the conditional distribution $p(\m|\v)$, which is a conditional Gaussian calculated from the joint distribution $p(\m,\v)$. The additive form associated with each $F_\v(y_{\bi_j},\beta)$ enables us to distribute the computation over multiple computers.

For binary data, we introduce a variational posterior $q(\z)$ and make the mean-field assumption that $q(\z)=\prod_j q(z_j)$. Following a similar derivation to the continuous case, we can obtain a tractable ELBO for binary data, $\log\big(p(\y,\Ucal|\B)\big) \ge L_2\big(\Ucal, \B, q(\v),q(\z)\big)$, where
\begin{align}
& L_2\big(\Ucal, \B, q(\v),q(\z)\big) = \log(p(\Ucal)) + \int q(\v) \log( \frac{p(\v|\B)}{q(\v)} )\d \v   \nonumber\\
&+\sum\nolimits_j q(z_j)\log(\frac{p(y_{\bi_j}|z_j)}{q(z_j)}) + \sum\nolimits_j  \int q(\v) \int q(z_j)F_\v(z_j, 1) \d z_j \d \v  \label{eq:elbo_b}. \raisetag{1in}
\end{align}

One can simply use the standard Expectation-maximization (EM) framework to optimize \eqref{eq:elbo_r} and \eqref{eq:elbo_b} for model inference, i.e., the E step updates the variational posteriors $\{q(\v), q(\z)\}$ and the M step updates the latent factors $\Ucal$, the inducing points $\B$ and the kernel parameters. However, the sequential E-M updates can not fully exploit the paralleling computing resources. Due to the strong dependencies between the E step and the M step, the sequential E-M updates may  take a large number of iterations to converge. Things become worse for binary case: in the E step, the updates of $q(\v)$ and $q(\z)$ are also dependent on each other, making a parallel inference even less efficient.
%For model inference, we can simply use the standard Expectation-maximization (EM) framework to optimize \eqref{eq:elbo_r} and \eqref{eq:elbo_b}, where the E step updates the variational posteriors $\{q(\v), q(\z)\}$ and the M step updates the latent factors $\Ucal$, the inducing points $\B$ and the kernel parameters. While each E or M step may be performed in parallel, there still exists numerous dependencies among them and so the converge can be slow. Things become worse for binary case: in the E step, the updates of $q(\v)$ and $q(\z)$ are dependent on each other, making a parallel inference even less efficient.

%\vspace{-0.1in}
\subsection{Tight and Parallelizable Variational Evidence Lower Bound}
%\vspace{-0.1in}
\cmt{
Given the tractable ELBOs in \eqref{eq:elbo_r} and \eqref{eq:elbo_b}, we can perform model inference simply following the standard Expectation-maximization (EM) framework, where the E step updates the variational posteriors $q(\v)$ (and $q(\z)$ for binary data) and the M step updates the latent factors $\Ucal$, the pseudo-inputs $\B$ and the kernel parameters. While each E or M step may be performed in parallel, there still exists numerous dependencies among them and so the converge can be slow. Things become worse for binary case: in the E step, the updates of $q(\v)$ and $q(\z)$ are dependent on each other, making a parallel inference even less efficient.
}
In this section,  we further derive tight(er) ELBOs that subsume the optimal variational posteriors for $q(\v)$ and $q(\z)$. Thereby we can avoid the sequential E-M updates to perform decoupled, highly efficient parallel inference. Moreover, the inference quality is very likely to be improved using tighter bounds. Due to the space limit, we only present key ideas and results here. Detailed discussions are given in Section 1 of the supplementary material.
%To address these issues, based on \eqref{eq:elbo_r} and \eqref{eq:elbo_b} we further derive tight(er) ELBOs that subsume the optimal variational posteriors for $q(\v)$ and $q(\z)$. Without explicitly updating the variational posteriors, not only do we save on computational costs, but the inference quality may also be improved due to a better evidence lower bound. Due to space limit, we present the key ideas and results here. The details are given in Section 1 of the supplementary material.

%Specifically, based on $L_1$ and $L_2$ in \eqref{eq:elbo_r} and \eqref{eq:elbo_b}, we use functional derivatives and convex conjugates to obtain the following tight ELBOs. The details are given in the Appendix.
\noindent \textbf{Tight ELBO for continuous tensors.} We take functional derivative of $L_1$ with respect to $q(\v)$ in \eqref{eq:elbo_r}. By setting the derivative to zero, we obtain the optimal $q(\v)$ (which is a Gaussian distribution) and then substitute it into $L_1$, manipulating the terms, we achieve the following tighter ELBO.

\begin{theorem}\label{th1}
For continuous data, we have
\begin{align}
 &\log\big(p(\y,\Ucal|\B)\big) \ge L_1^*(\Ucal, \B) =  \frac{1}{2}\log |\K_{BB}| - \frac{1}{2}\log|\K_{BB} + \beta\A_1|-\frac{1}{2}\beta a_2- \frac{1}{2}\beta a_3 \nonumber \\
 & + \frac{\beta}{2}\tr(\K_{BB}^{-1}\A_1)-\frac{1}{2}\sum_{k=1}^K \| \U^{(k)} \|^2_F + \frac{1}{2}\beta^2 \a_4^\top(\K_{BB}+\beta\A_1)^{-1}\a_4 + \frac{N}{2}\log(\frac{\beta}{2\pi}), \label{eq:t_lb_r}
 \end{align}
 where $\|\cdot \|_F$ is Frobenius norm, and
  \begin{align}
  \A_1 &= \sum\nolimits_j k(\B, \x_{\bi_j}) k(\x_{\bi_j}, \B), \;\;\;\;\;\;\;\;\;\; a_2 = \sum\nolimits_j y_{\bi_j}^2, \nonumber \\
  a_3 &= \sum\nolimits_j k(\x_{\bi_j}, \x_{\bi_j}), \;\;\;\;\;\;\;\;\;\;  \a_4 = \sum\nolimits_j k(\B, \x_{\bi_j}) y_{\bi_j} \nonumber.
  \end{align}
\end{theorem}

\noindent \textbf{Tight ELBO for binary tensors.} The binary case is more difficult because $q(\v)$ and $q(\z)$ are coupled together (see \eqref{eq:elbo_b}). We use the following steps: we first fix $q(\z)$ and plug the optimal $q(\v)$ in the same way as the continuous case. Then we obtain an intermediate ELBO $\hat{L}_2$ that only contains $q(\z)$. However, a quadratic term in $\hat{L}_2$ , $\frac{1}{2}(\K_{BS}\expt{\z})^\top(\K_{BB}+\A_1)^{-1}(\K_{BS}\expt{\z})$, intertwines all $\{q(z_j)\}_j$ in $\hat{L}_2$, making it infeasible to analytically derive or parallelly compute the optimal $\{q(z_j)\}_j$. To overcome this difficulty, we exploit the convex conjugate of the quadratic term to introduce an extra variational parameter $\blambda$ to decouple the dependences between $\{q(z_j)\}_j$. After that, we are able to derive the optimal $\{q(z_j)\}_j$ using functional derivatives and to obtain the following tight ELBO.
\begin{theorem}\label{th2}
For binary data, we have
\begin{align}
 &\log\big(p(\y,\Ucal|\B)\big) \ge L_2^*(\Ucal, \B, \blambda) =  \frac{1}{2}\log |\K_{BB}| - \frac{1}{2}\log|\K_{BB} + {\A}_1| - \frac{1}{2} {a}_3\nonumber \\
&+ \sum_j \log\big(\Phi((2y_{\bi_j}-1)\blambda^\top k(\B, \x_{\bi_j}))\big)- \frac{1}{2}\blambda^\top\K_{BB}\blambda  + \frac{1}{2}\tr(\K_{BB}^{-1}{\A}_1) \nonumber \\
&-\frac{1}{2}\sum_{k=1}^K \| \U^{(k)} \|^2_F \label{eq:t_lb_b}
\end{align}
where $\Phi(\cdot)$ is the cumulative distribution function of the standard Gaussian.
\end{theorem}
As we can see, due to the additive forms of the terms in $L_1^*$ and $L_2^*$, such as $\A_1$, $a_2$, $a_3$ and $\a_4$, the computation of the tight ELBOs and their gradients can be efficiently performed in parallel. The derivation of the full gradient is given in Section 2 of the supplementary material.
%The derivation of the full gradient is lengthy and standard, and we omit it here to save space.
%Interested readers are referred to the supplementary material.

%\vspace{-0.1in}
\subsection{Distributed Inference on Tight Bound}
%\vspace{-0.05in}
\subsubsection{Distributed Gradient-based Optimization}
%\vspace{-0.05in}
Given the tighter ELBOs in \eqref{eq:t_lb_r} and \eqref{eq:t_lb_b}, we develop a distributed algorithm to optimize the latent factors $\Ucal$, the inducing points $\B$, the variational parameters $\blambda$ (for binary data) and the kernel parameters. We distribute the computations over multiple computational nodes (\map step) and then collect the results to calculate the ELBO and its gradient (\reduce step). A standard routine, such as gradient descent and L-BFGS, is then used to solve the optimization problem.

For binary data, we further find that $\blambda$ can be updated with a simple fixed point iteration:
\begin{align}
\blambda^{(t+1)} = (\K_{BB} + \A_1)^{-1} (\A_1 \blambda^{(t)} + \a_5)\label{eq:fix_point}
\end{align}
where  $\a_5 = \sum_j  k(\B, \x_{\bi_j})(2y_{\bi_j}-1)\frac{\N\big(k(\B, \x_{\bi_j})^\top \blambda^{(t)}|0,1\big)}{\Phi\big((2y_{\bi_j}-1) k(\B, \x_{\bi_j})^\top \blambda^{(t)}\big)}$.

\cmt{
\begin{align}
\blambda^{(t+1)} = \D^{-1} \big(\sum\nolimits_j \k_{Bj}(w_j^{(t)} +  \k_{Bj}^\top \blambda^{(t)})\big)\label{eq:fix_point}
\end{align}
 where $\D \triangleq \K_{BB} + \A_1$, $\k_{Bj} \triangleq k(\B, \x_{\bi_j})$ and
 $w_j^{(t)} = (2y_{\bi_j}-1)\frac{\N(\k_{Bj}^\top \blambda^{(t)}|0,1)}{\Phi((2y_{\bi_j}-1) \k_{Bj}^\top \blambda^{(t)})}$.
}
 Apparently, the updating can be efficiently performed in parallel (due to the additive structure of $\A_1$ and $\a_5$). Moreover, the convergence is guaranteed by the following lemma. The proof is given in Section 3 of the supplementary material.
 \begin{lem}\label{lem1}
 Given $\Ucal$ and $\B$, we have $L_2^*(\Ucal, \B, \blambda^{t+1}) \ge L_2^*(\Ucal, \B, \blambda^{t})$ and the fixed point iteration \eqref{eq:fix_point} always converges.
\end{lem}

In our experience, the fixed-point iterations are much more efficient than general search strategies (such as line-search) to identity an appropriate step length along the gradient direction. To use it,  before we calculate the gradients with respect to $\Ucal$ and $\B$, we first optimize $\blambda$ using the fixed point iteration (in an inner loop). In the outer control, we then employ gradient descent or L-BFGS to optimize $\Ucal$ and $\B$. This will lead to an even tighter bound for our model: $L_2^{**}(\Ucal,\B) = \max\nolimits_{\blambda} L_2^*(\Ucal, \B, \blambda) = \max\nolimits_{q(\v),q(\z)} L_2(\Ucal, \B, q(\v), q(\z))$. Empirically, this converges must faster than feeding the optimization algorithms with $\partial \lambda$, $\partial \Ucal$ and $\partial \B$ altogether.

%\subsubsection{Efficient Map-Reduce without (key,value) Pairs}
%\vspace{-0.1in}
\subsubsection{Key-Value-Free \mapreduce}
%\vspace{-0.05in}
In this section we present the detailed design of \mapreduce procedures to fulfill our distributed inference. Basically, we first allocate a set of tensor entries $S_t$ on each \mapper $t$ such that the corresponding components of the ELBO and the gradients are calculated. Then the \reducer aggregates local results from each \mapper to obtain the integrated, global ELBO and gradient.

We first consider the standard (key-value) design. For brevity, we take the gradient computation for the latent factors as an example. For each tensor entry $\bi$ on a \mapper, we calculate the corresponding gradients $\{ \partial \u^{(1)}_{i_1}, \ldots \partial \u^{(K)}_{i_K} \} $ and then send out the key-value pairs $\{( k, i_k) \rightarrow \partial \u^{(k)}_{i_k}\}_k$, where the key indicates the mode and the index of the latent factors. The \reducer aggregates gradients with the same key to recover the full gradient with respect to each latent factor.

Although the (key-value) \mapreduce has been successfully applied in numerous applications,  it relies on an expensive data shuffling operation: the \reduce step has to sort the \mappers ' output by the keys before aggregation. Since the sorting is usually performed on disk due to significant data size, intensive disk I/Os and network communications will become serious computational overheads.
%We now discuss the design of detailed \map and \reduce procedures. For each \mapper $t$,  a set of tensor entries $S_t$ are allocated and the local ELBO and its gradient are calculated. For each entry $\bi \in S_t$, %because $\x_{\bi} = [\u^{(1)}_{i_1}, \ldots, \u^{(K)}_{i_K}]$,
%we calculate the gradients of associated latent factors $\{ \partial \u^{(1)}_{i_1}, \ldots \partial \u^{(K)}_{i_K} \} $ and then send out the key-value pairs $\{( k, i_k) \rightarrow \partial \u^{(k)}_{i_k}\}_k$, where the key indicates the mode and the index of the latent factors. The \reduce step aggregates gradients with the same key value to recover the full gradient with respect to each latent factor. This design is quite standard, but it relies on an expensive data shuffling operation: the \reduce step has to sort the \mappers ' output by key values before aggregation. Since the sorting is usually performed on disk due to the large amount of data, the intensive disk I/Os and network communications will  make the process less efficient.p
To overcome this deficiency, we devise a key-value-free \map-\reduce scheme to avoid on-disk data shuffling operations. Specifically, on each \mapper, a complete gradient vector is maintained for all the parameters, including $\Ucal$, $\B$ and the kernel parameters. However, only relevant components of the gradient, as specified by the tensor entries allocated to this \mapper , will be updated. %Then the gradients are calculated according to the allocated tensor entries and stored in the complete gradient vector.
After updates, each \mapper will then send out the full gradient vector, and the \reducer will simply sum them up together to obtain a global gradient vector without having to perform any extra data sorting. Note that a similar procedure can also be used to perform the fixed point iteration for $\blambda$ (in binary tensors).

Efficient \mapreduce systems, such as \spark~\citep{zaharia2012resilient}, can fully optimize the non-shuffling \map and \reduce, where most of the data are buffered in memory and disk I/Os are circumvented to the utmost; by contrast, the performance with data shuffling degrades severely~\citep{davidson2013optimizing}. This is verified in our evaluations: on a small tensor of size $100 \times 100 \times 100$, our key-value-free \mapreduce gains $30$ times speed acceleration over the traditional key-value process.  Therefore, our algorithm can fully exploit the memory-cache mechanism to achieve fast inference.
%\vspace{-0.1in}
\cmt{
\begin{algorithm}
\caption{Distributed variational inference  ($\y =\{y_{\bi_1}, \ldots, y_{\bi_N}\}$)}
\begin{algorithmic}[1]
\STATE Allocate the tensor entries $\y$ on $T$ \mappers.
\STATE Initialize $\Ucal$, $\B$ and the kernel parameters $\btheta$.
\STATE Optimize $\Ucal$, $\B$ and $\btheta$ by feeding a standard optimization routine, such as gradient descent or L-BFGS, with \textit{DistELBOandGradient($\Ucal$, $\B$, $\btheta$)}.
\STATE Return $\Ucal$, $\B$ and $\btheta$.
\end{algorithmic}
\end{algorithm}
%\vspace{-0.25in}
\begin{algorithm}
\caption{DistELBOandGradient($\Ucal$, $\B$, $\btheta$)}
\begin{algorithmic}[1]
	\IF{Binary tensor}
    	\STATE Optimize $\blambda$ using fixed point iteration \eqref{eq:fix_point} with  \map-\reduce.
    \ENDIF
    \FOR {each \mapper\, {\bfseries parallel} }
    	\STATE Calculate the local ELBO and the gradients according to \eqref{eq:t_lb_r} or \eqref{eq:t_lb_b}.
    	\STATE Send out a whole gradient vector and the local ELBO.
    \ENDFOR
    \STATE \reduce task: Sum the output from all the \mappers.
    \STATE Return the global ELBO and the full gradients with respect to $\Ucal$, $\B$ and  $\btheta$.
\end{algorithmic}\label{alg:distGD}
\end{algorithm}
}
%\vspace{-0.15in}
\subsection{Algorithm Complexity}
%\vspace{-0.1in}
Suppose we use $N$ tensor entries for training, with $p$ inducing points and $T$ {\mapper}, the time complexity for each \mapper node is $O(\frac{1}{T}p^2N)$.
%, because the calculation of the EBLO and its gradient need to invert the covariance matrix over inducing points, $k(\B,\B)$.
Since $p \ll N$ is a fixed constant ($p=100$ in our experiments), the time complexity is linear in the number of tensor entries. The space complexity for each \mapper node is $O(\sum_{j=1}^K m_j r_j + p^2 + \frac{N}{T}K)$, in order to store the latent factors, their gradients, the covariance matrix on inducing points,  and the indices of the latent factors for each tensor entry. Again, the space complexity is linear in the number of tensor entries.
In comparison, \InfTucker utilizes the Kronecker-product properties to calculate the gradients and has to perform eigenvalue decomposition of the covariance matrices in each tensor mode. Therefor it has a higher time and space complexity (see ~\citep{XuYQ12} for details) and is not scalable to large dimensions.

%As a comparison, the time complexity of \InfTucker is $O(\sum_k d_k^3 + d_k\prod_t d_t)$, including the eigen-decomposition of the covariance matrices in each mode and the tensor-matrix products, based on the Kronecker-product properties, to calculate the gradients (see~\citep{XuYQ12} for details). When the dimension of one mode is large, the eigen-decomposition is infeasible. If we assume that $d_1=\ldots=d_K$, the time complexity of \InfTucker is simplified to $O(N^{1+\frac{1}{K}})$ where $N = \prod_t d_t$, which is higher than our algorithm using all the tensor entries. Further, because \InfTucker needs to store the covariance matrices in every mode while ours only stores the one on a small set of inducing points, it usually possesses a bigger memory storage and is not scalable for large dimensions.

% Furthermore, the local computation is one efficiently by coupling stochastic gradient descent with variational inference.
%\vspace{-0.15in}
\section{Related work}
%\vspace{-0.15in}
Classical tensor factorization models include Tucker~\citep{Tucker66} and CP~\citep{Harshman70parafac}, based on which many excellent works have been proposed \citep{ShashuaH05,Chu09ptucker,sutskever2009modelling,acar2011scalable,hoff_2011_csda,YangDunson13Tensor,RaiDunson2014,sun2015provable,hu2015zero,raiscalable}. To deal with big data, several distributed factorization algorithms have been recently developed, such as GigaTensor~\citep{kang2012gigatensor} and DFacTo~\citep{choi2014dfacto}.
Despite the widespread success of these methods, their underlying multilinear factorization structure may limit their capability to capture more complex, nonlinear relationship in real-world applications. Infinite Tucker decomposition~\citep{XuYQ12}, and its distributed or online extensions~\citep{zhe2013dintucker,zhe2015scalable} address this issue by modeling tensors or subtensors via a tensor-variate Gaussian process (TGP). However, these methods may suffer from the extreme sparsity in real-world tensor data, because the Kronecker-product structure in the covariance of TGP requires modeling the entire tensor space no matter the elements are meaningful (non-zeros) or not. By contrast, our flexible GP factorization model eliminates the Kronecker-product restriction and can model an arbitrary subset of tensor entries. In theory, all such nonlinear factorization models belong to the random function prior models~\citep{LloydOGR12randomgraph} for exchangeable multidimensional arrays.

Our distributed variational inference algorithm is based on sparse GP~\citep{quinonero2005unifying}, an efficient approximation framework to scale up GP models. Sparse GP uses a small set of inducing points to break the dependency between random function values. %Many works have been proposed for sparse GP, such as SoR~\citep{silverman1985some}, DTC~\citep{seeger2003fast} and FITC~\citep{snelson2005sparse}. An excellent review of these works are given by \citet{quinonero2005unifying}.
Recently, \citet{titsias2009variational} proposed a variational learning framework for sparse GP, based on which \citet{gal2014distributed} derived a tight variational lower bound for distributed inference of GP regression and GPLVM~\citep{lawrence2004gaussian}. The derivation of the tight ELBO in our model for continuous tensors is similar to \citep{gal2014distributed}. However, the gradient calculation is substantially different, because the input to our GP factorization model is the concatenation of the latent factors. Many tensor entries may partly share the same latent factors, causing a large amount of key-value pair to be sent during the distributed gradient calculation. This will incur an expensive data shuffling procedure that takes place on disk. To improve the computational efficiency, we develop a non-key-value \map-\reduce to avoid data shuffling and fully exploit the memory-cache mechanism in efficient \mapreduce systems. This strategy is also applicable to  other \map-\reduce based learning algorithms. In addition to continuous data, we also develop a tight ELBO for binary data on optimal variational posteriors. By introducing $p$ extra variational parameters with convex conjugates ($p$ is the number of inducing points), our inference can be performed efficiently in a distributed manner, which avoids explicit  optimization on a large number of variational posteriors for the latent tensor entries and inducing targets. Our method can also be useful for GP classification problem.

\section{Experiments}
%\vspace{-0.1in}
\subsection{Evaluation on Small Tensor Data}
%\vspace{-0.1in}
For evaluation, we first compared our method with various existing tensor factorization methods. To this end, we used four small real datasets where all methods are computationally feasible:  (1) \textit{Alog}, a real-valued tensor of size $200\times 100 \times 200$, representing a three-way interaction (user, action, resource) in a file access log. It contains $0.33\%$ nonzero entries.(2) \textit{AdClick}, a real-valued tensor of size $80 \times 100 \times 100$, describing (user, publisher, advertisement) clicks for online advertising. It contains $2.39\%$ nonzero entries. (3) \textit{Enron}, a binary tensor extracted from the Enron email dataset (\url{www.cs.cmu.edu/~./enron/})  depicting the three-way relationship (sender, receiver, time). It contains  $203\times 203 \times 200$ elements, of which $0.01\%$ are nonzero.
(4) \textit{NellSmall}, a binary tensor extracted from the NELL knowledge base (\url{rtw.ml.cmu.edu/rtw/resources}), of size $295 \times 170 \times  94$. It depicts the knowledge predicates (entity, relationship, entity). The data set contains  $0.05\%$ nonzero elements.

%\textbf{Datasets}. First, four small datasets were used to examine the predictive performance of the proposed  model: (1) \textit{Alog}, a real-valued tensor of size $200\times 100 \times 200$, was extracted from an access log from a file management system. It describes a three-way interaction (user, action, resource) and  contains $0.33\%$ nonzero entries.(2) \textit{AdClick}, a real-valued tensor of size $80 \times 100 \times 100$, describing (user, publisher, advertisement) clicks. It was extracted from an online advertisement click log and contains $2.39\%$ nonzero entries. (3) \textit{Enron}, a binary tensor extracted from the Enron email dataset (\url{www.cs.cmu.edu/~./enron/}) and  depicting three-way relationships (sender, receiver, time). It contains  $203\times 203 \times 200$ elements, of which $0.01\%$ are nonzero.
%(4) \textit{NellSmall}, a binary tensor extracted from the NELL knowledge base (\url{rtw.ml.cmu.edu/rtw/resources}), of size $295 \times 170 \times  94$. It depicts the knowledge predicts (entity, relationship, entity). The data contains  $0.05\%$ nonzero elements.

 We compared with CP, nonnegative CP (NN-CP)~\citep{ShashuaH05}, high order SVD (HOSVD)~\citep{Lathauwer00HOSVD}, Tucker, infinite Tucker (\InfTucker) \citet{XuYQ12} and its extension (\InfTuckerEx) which uses the Dirichlet process mixture (DPM) prior to model latent clusters and local TGP to perform scalable, online factorization~\citep{zhe2015scalable}. Note that \InfTucker and \InfTuckerEx are nonlinear factorization approaches.

For testing, we used the same setting as in~\citep{zhe2015scalable}. All the methods were evaluated via a 5-fold cross validation. The nonzero entries were randomly split into $5$ folds: $4$ folds were used for training and the remaining non-zero entries and $0.1\%$ zero entries were used for testing so that the number of non-zero entries is comparable to the number of zero entries. By doing this, zero and nonzero entries are treated equally important in testing, and so the evaluation will not be dominated by large portion of zeros. For \InfTucker and \InfTuckerEx, we carried out extra cross-validations to select the kernel form (\eg RBF, ARD and Matern kernels) and the kernel parameters. For \InfTuckerEx, we randomly sampled subtensors and tuned the learning rate following ~\citep{zhe2015scalable}.
%For \InfTuckerEx, the variational truncation level for the DPM prior is set to the dimension in each mode divided by $10$; for online learning of local TGP, we randomly sampled $500$ subtensors on each dataset; the size of subtensor is  $40 \times 40 \times 40$; the learning rate was tuned from the range $\{10^{-5}, 10^{-6}, 10^{-7}, 10^{-8}\}$; moreover, for the bagging prediction of \InfTuckerEx,  we randomly sampled $10$ subtensors, each of size $40 \times 40 \times 40$;  all these settings are consistent with \citep{zhe2015scalable}.
 For our model, the number of inducing points was set to $100$, and we used a balanced training set generated as follows: in addition to nonzero entries, we randomly sampled the same number of zero entries and made sure that they would not overlap with the testing zero elements.

Our model used ARD kernel and the kernel parameters were estimated jointly with the latent factors. Thus, the expensive parameter selection procedure was not needed.
We implemented our distributed inference algorithm with two optimization frameworks, gradient descent and L-BFGS (denoted by Ours-GD and Ours-LBFGS respectively).
For a comprehensive evaluation, we also examined CP on balanced training entries generated in the same way as our model,  denoted by CP-2.
The mean squared error (MSE) is used to evaluate predictive performance on \textit{Alog} and \textit{Click} and area-under-curve (AUC) on \textit{Enron} and \textit{Nell}. The averaged results from the $5$-fold cross validation are reported.

 Our model achieves a higher prediction accuracy than \InfTucker, and a better or comparable accuracy than \InfTuckerEx (see Figure \ref{fig:small-pred}). A $t$-test shows that our model  outperforms \InfTucker significantly ($p<0.05$) in almost all situations.
Although \InfTuckerEx uses the DPM prior to improve factorization, our model still obtains significantly better predictions on \textit{Alog} and \textit{AdClick} and comparable or better performance on \textit{Enron} and \textit{NellSmall}. This might be attributed to the flexibility of our model in  using balanced training entries to prevent the learning bias (toward numerous zeros). Similar improvements can be observed from CP to CP-2. %The improvement is more obvious in continuous data. This is reasonable, because in continuous data, the task is to recover the quantitative values of each tensor entry and thus the learning is more easily dominated by massive zeros, while in binary data, the prediction is qualitative and only needs a good classification boundary.
Finally, our model outperforms all the remaining methods, demonstrating the advantage of our nonlinear factorization approach.
\subsection{Scalability Analysis}
To examine the scalability of the proposed distributed inference algorithm, we used the following large real-world datasets: (1) ACC, A real-valued tensor describing three-way interactions (user, action, resource) in a code repository management system \citep{zhe2015scalable}. The tensor is of size $3K \times 150 \times 30K$, where  $0.009\%$ are nonzero. (2) DBLP: a binary tensor depicting a three-way bibliography relationship (author, conference, keyword) \citep{zhe2015scalable}. The tensor was extracted from DBLP database\cmt{\footnote{\url{http://dblp.uni-trier.de/xml/}}} and contains $10K \times 200 \times 10K$ elements, where $0.001\%$ are nonzero entries.
(3) NELL: a binary tensor representing the knowledge predicates, in the form of (entity, entity, relationship) \citep{zhe2013dintucker}. %The data was obtained from the 'Read the Web' project \citep{carlson2010toward}.
The tensor size is $20K \times 12.3K \times 280$ and $0.0001\%$ are nonzero.

The scalability of our distributed inference algorithm was examined with regard to the number of machines on ACC dataset. The number of latent factors was set to 3. %Then we ran 10 iterations of our inference algorithm based on gradient descent.
We ran our algorithm using the gradient descent. The results are shown in Figure \ref{fig:large}(a). The Y-axis shows the reciprocal of the running time multiplied by a constant---which corresponds to the running speed. As we can see, the speed of our algorithm scales up linearly to the number of machines.
\begin{figure*}
\centering
\begin{tabular}[c]{cccc}
\multicolumn{4}{c}{\includegraphics[scale=0.25]{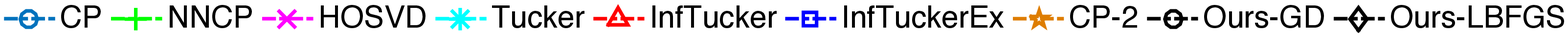} }
\vspace{-0.1in}
\\
\subfigure[\textit{Alog}]{
\includegraphics[width=0.2\textwidth]{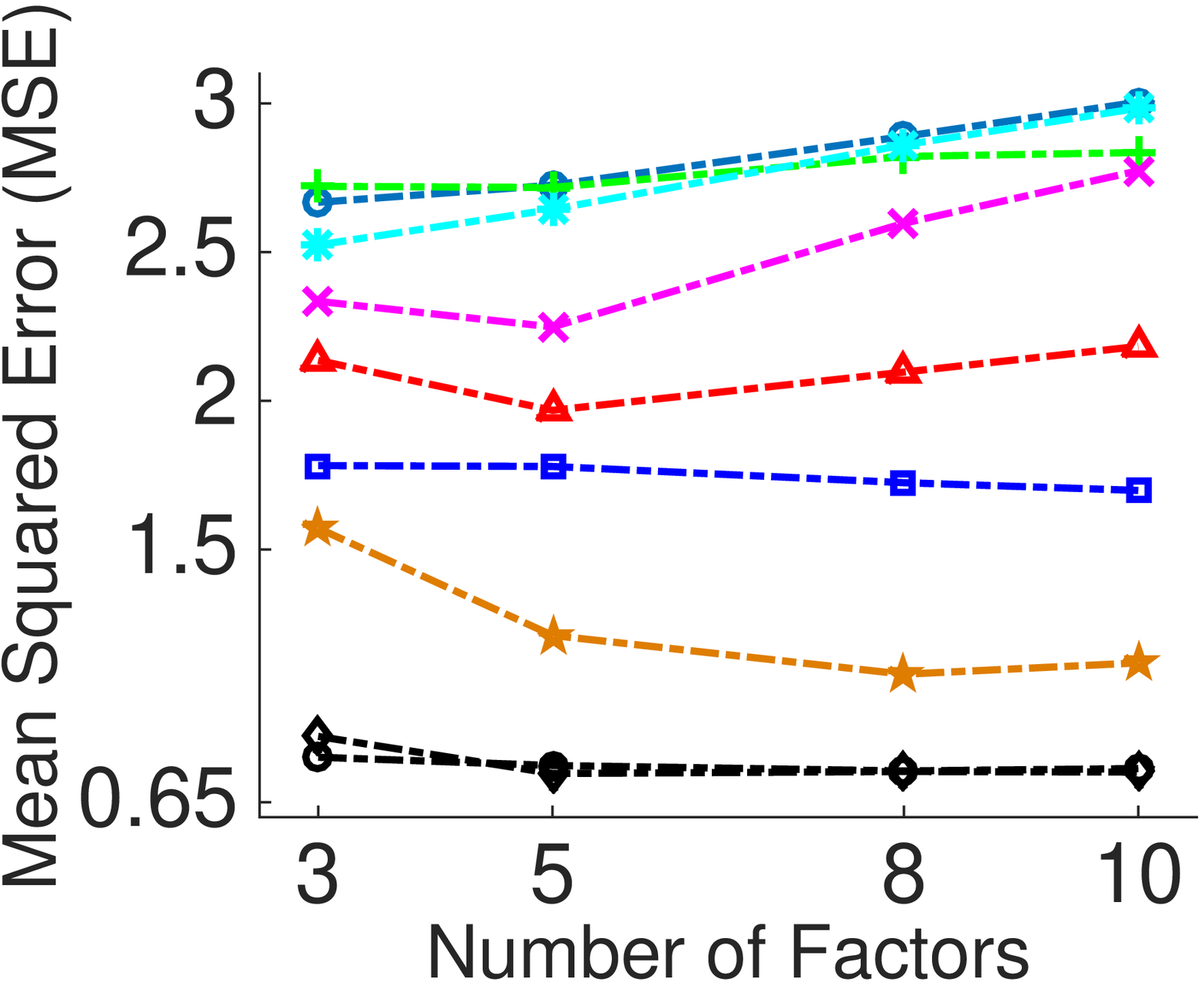}
}
\setcounter{subfigure}{1}
 &
\subfigure[\textit{AdClick}]{
\includegraphics[width=0.2\textwidth]{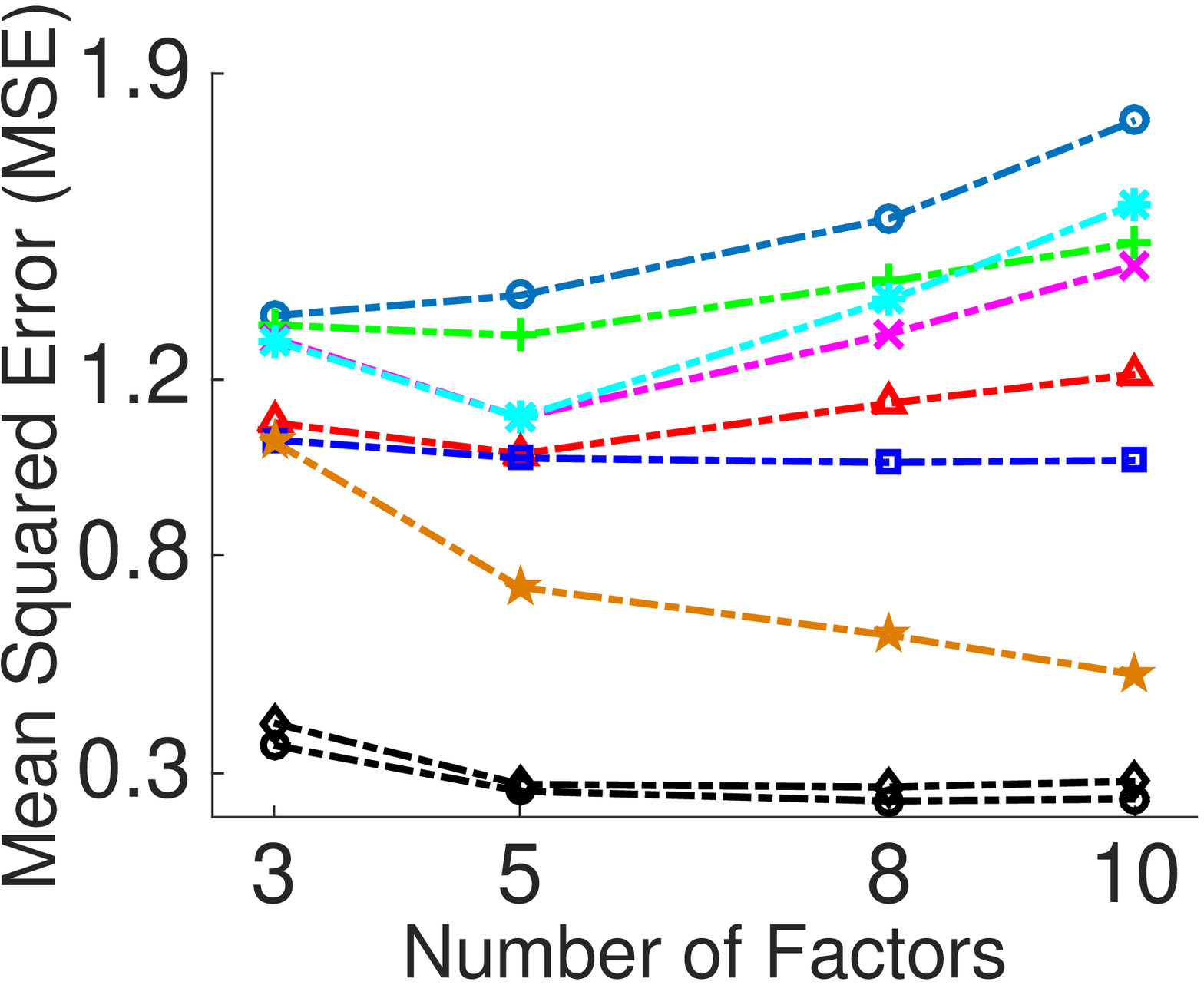}
} &
\subfigure[\textit{Enron}]{
\includegraphics[width=0.2\textwidth]{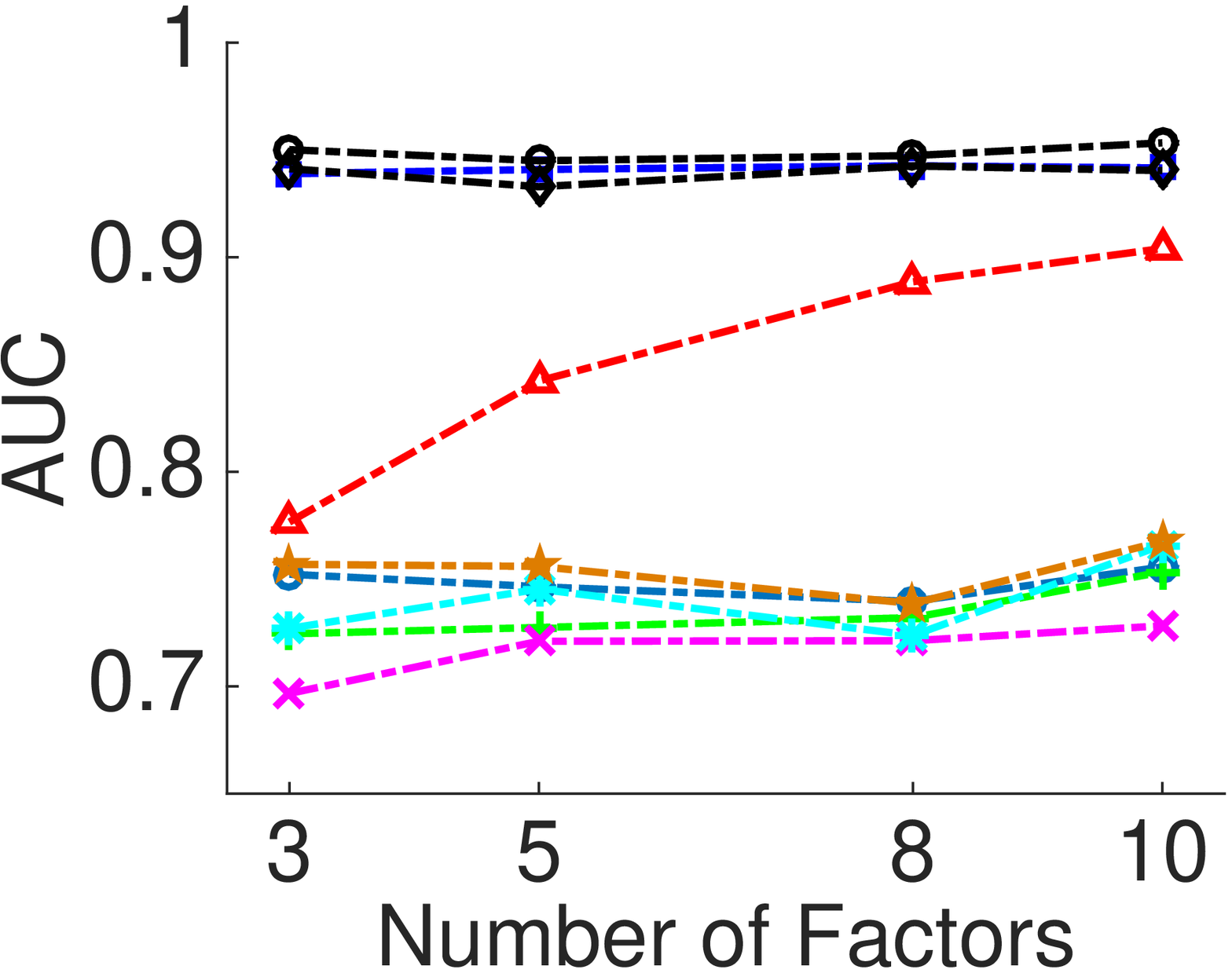}
} &
\subfigure[\textit{NellSmall}]{
\includegraphics[width=0.2\textwidth]{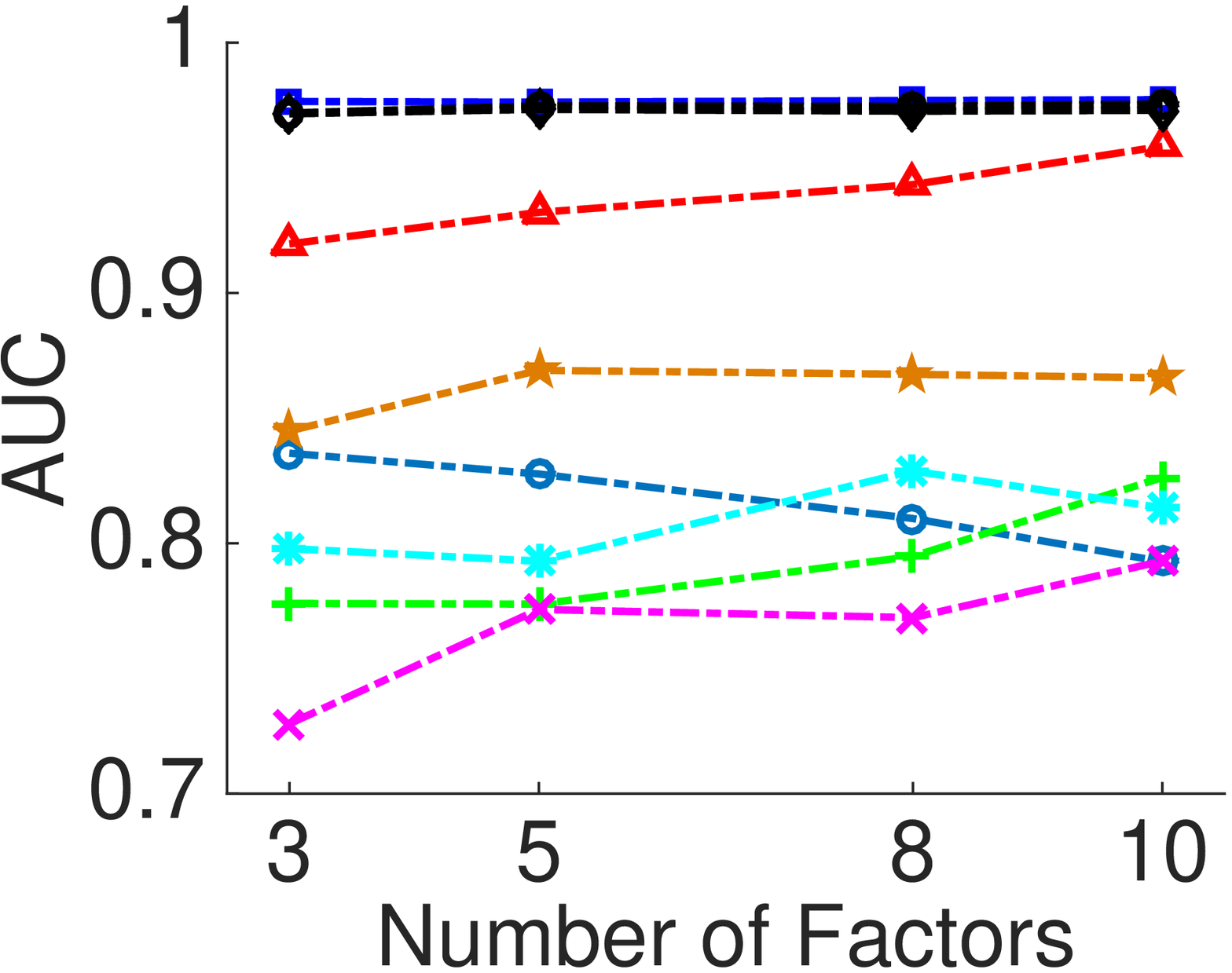}
}
\end{tabular}
%\vspace{-0.1in}
\caption{The prediction results on small datasets. The results are averaged over 5 runs.}
\label{fig:small-pred}
%\vspace{-0.2in}
\end{figure*}
\begin{figure*}
\centering
\begin{tabular}[c]{cccc}
\setcounter{subfigure}{0}
\subfigure[{Scalability}]{
	\includegraphics[width=0.2\textwidth]{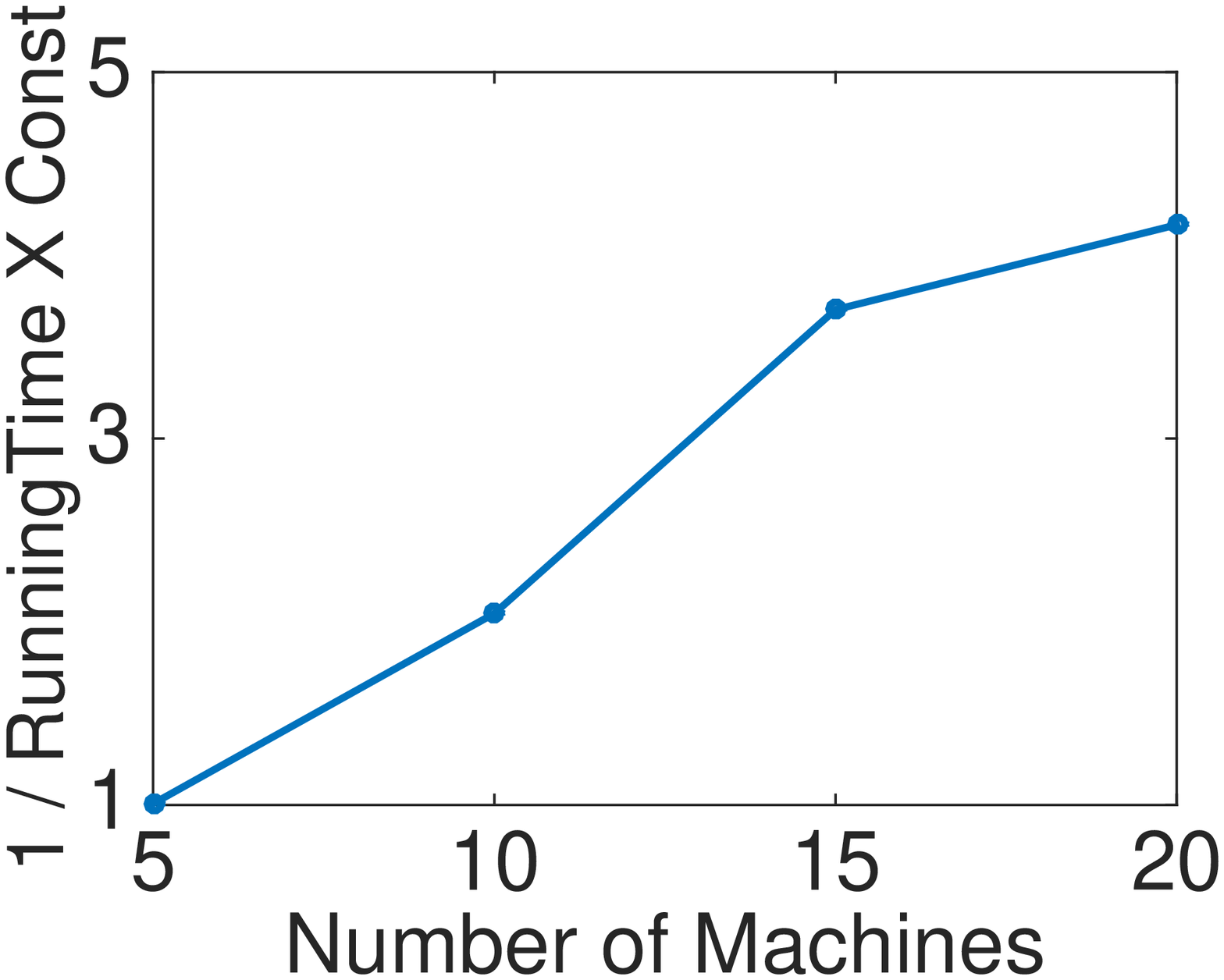}
}
&
\subfigure[{ACC}]{
\includegraphics[width=0.2\textwidth]{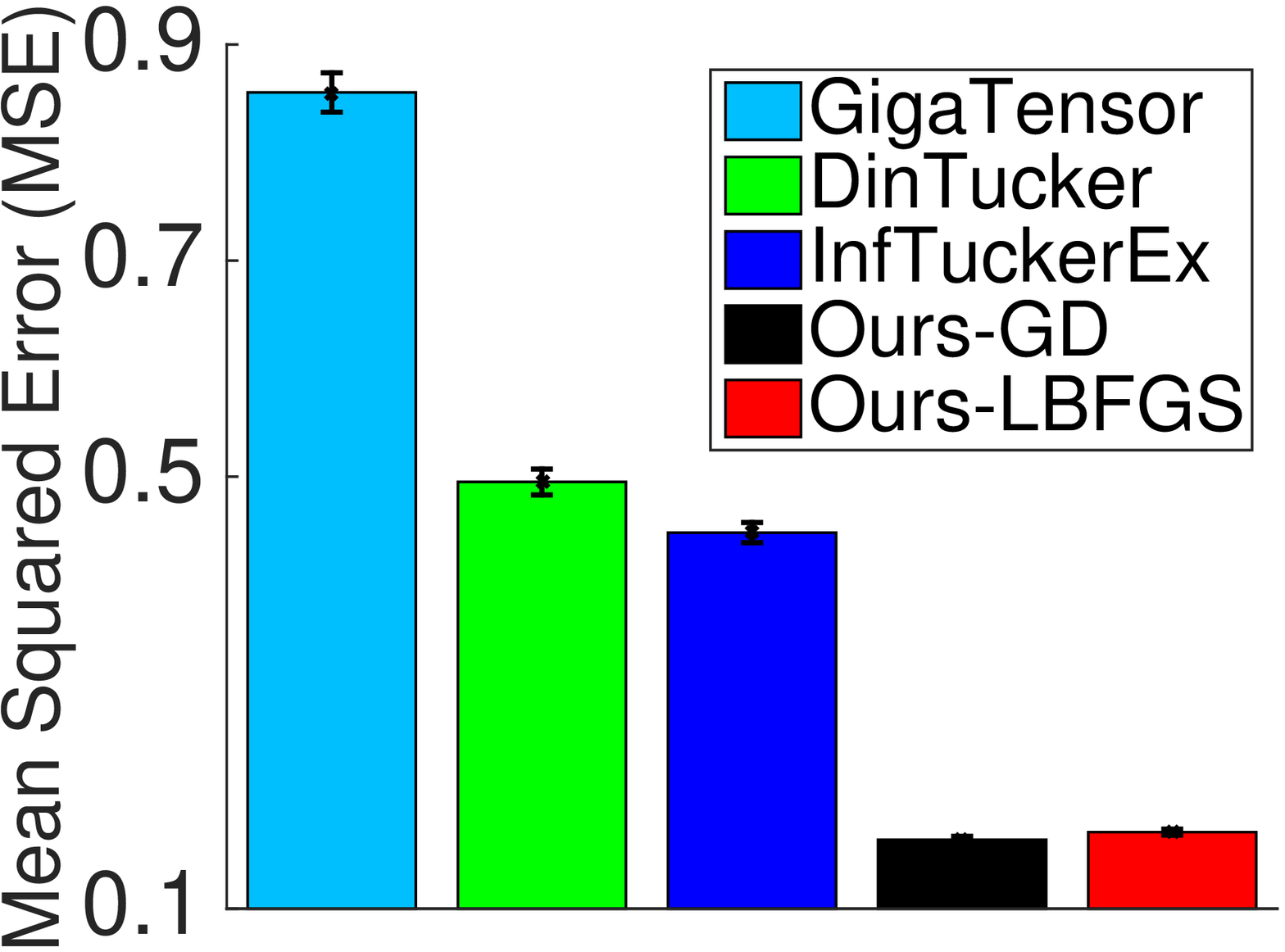}
}
&
\subfigure[{DBLP}]{
\includegraphics[width=0.2\textwidth]{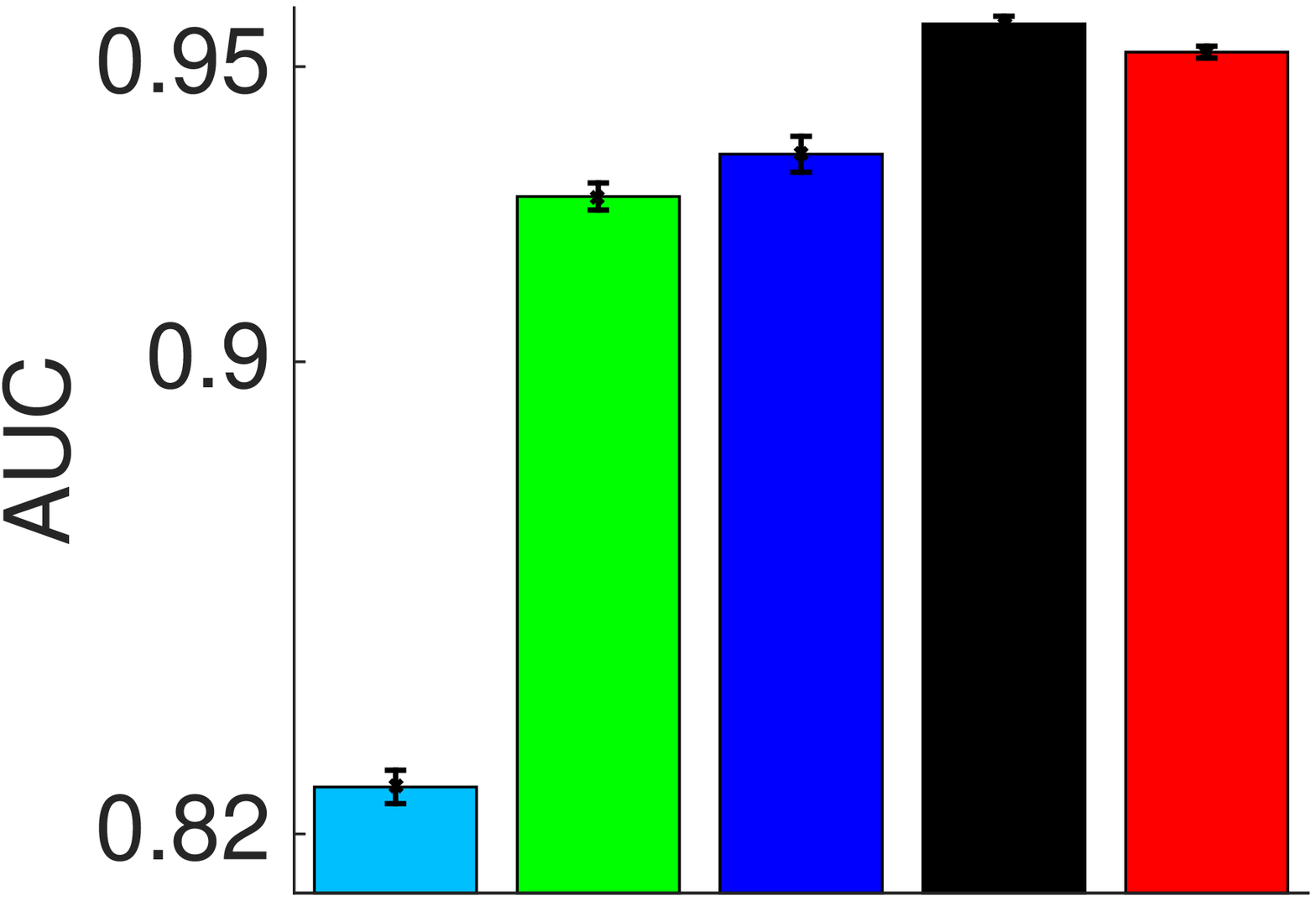}
} &
\subfigure[{NELL}]{
\includegraphics[width=0.2\textwidth]{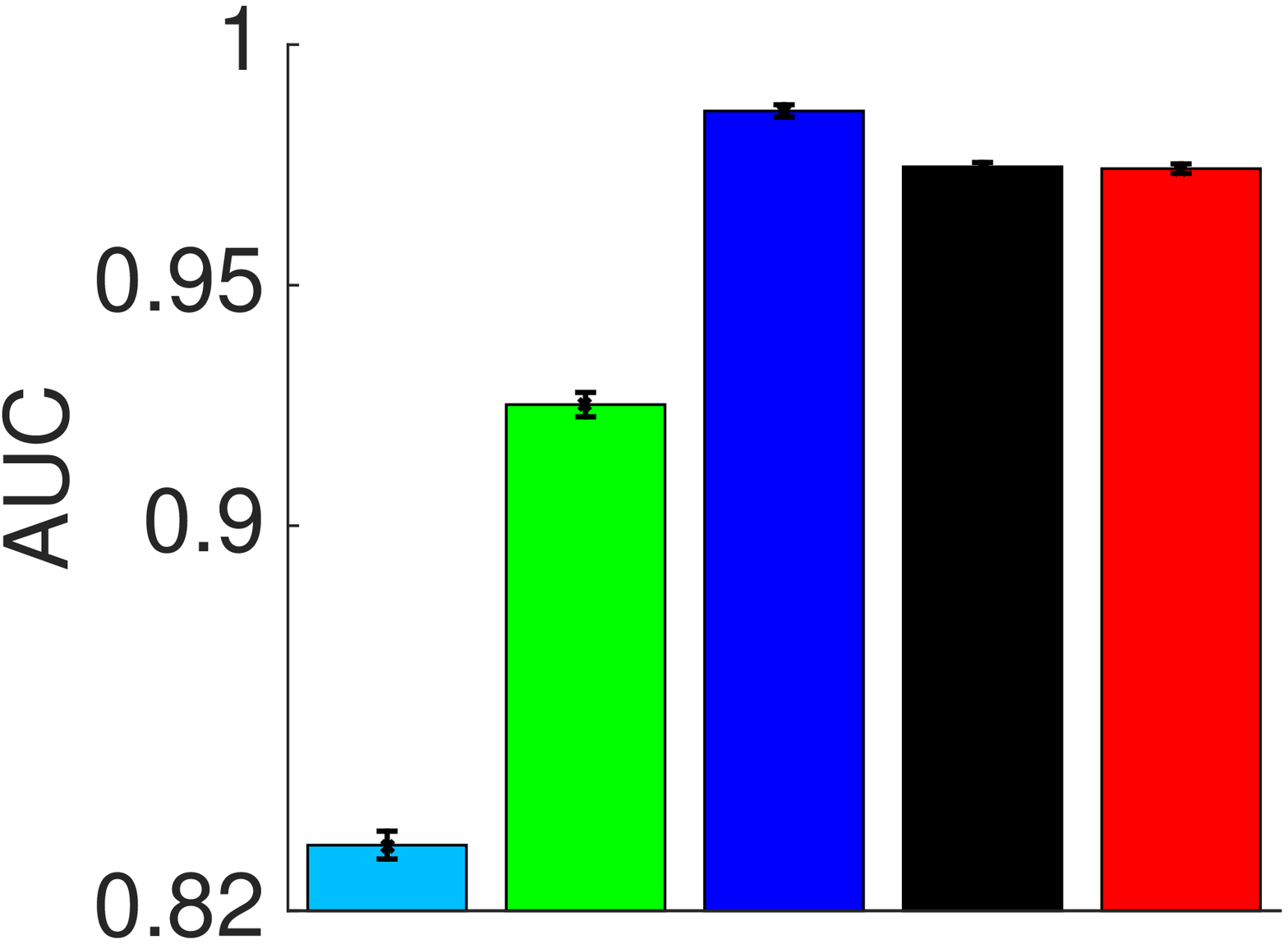}
}
\end{tabular}
\caption{Prediction accuracy (averaged on $50$ test datasets) on large tensor data and the scalability. }
\label{fig:large}
\end{figure*}
\subsection{Evaluation on Large Tensor Data}
We then compared our approach with three state-of-the-art large-scale tensor factorization methods: GigaTensor~\citep{kang2012gigatensor}, Distributed infinite Tucker decomposition (DinTucker)~\citep{zhe2013dintucker}, and \InfTuckerEx~\citep{zhe2015scalable}. Both GigaTensor and DinTucker are developed on \hadoop, while \InfTuckerEx uses online inference. Our model was implemented on \spark. We ran Gigatensor, DinTucker and our approach on a large YARN cluster and \InfTuckerEx on a single computer. % We used the original GigaTensor software package and its default settings. We ran GigaTensor on a Hadoop cluster with $16$ computers and our online algorithm on a single computer. %We chose RBF kernel for our model.

We set the number of latent factors to $3$ for ACC and DBLP data set, and $5$ for NELL data set. Following the settings in \citep{zhe2015scalable,zhe2013dintucker}, we randomly chose $80\%$ of nonzero entries for training, and then sampled $50$ test data sets from the remaining entries. For ACC and DBLP, each test data set comprises $200$ nonzero elements and $1,800$ zero elements; for NELL, each test data set contains $200$ nonzero elements and $2,000$ zero elements. The running of GigaTensor was based on the default settings of the software package.
For DinTucker and \InfTuckerEx, we randomly sampled subtensors for distributed or online inference. The parameters, including the number and size of the subtensors and the learning rate, were selected in the same way as~\citep{zhe2015scalable}. The kernel form and parameters were chosen by a cross-validation on the training tensor.
%The size of the subtensors was chosen from $\{100\times 100 \times 100, 125\times 125 \times 125, 150\times 150\times 150\}$. The number of subtensors were set to the one that can cover the same quantity of entries in the whole tensor. The learning rate was tuned from $\{10^{-6},10^{-7},10^{-8},10^{-9}\}$; the bagging prediction used 10 subtensors, each of size $100 \times 100 \times 100$;
%For DinTucker and \InfTuckerEx, the kernel form and kernel parameters were chosen by cross-validation on the training tensor.
For our model, we used the same setting as in the small data.
%For our model, we used ARD kernel, with kernel parameters jointly estimated during inference; we used a balanced training set, with zero entries randomly sampled with the same size of the nonzero training entries.
We set $50$ \mappers for GigaTensor, DinTucker and our model.

Figure \ref{fig:large}(b)-(d) shows the predictive performance of all the methods. We observe that our approach  consistently
outperforms GigaTensor and DinTucker on all the three datasets; our approach outperforms \InfTuckerEx on ACC and DBLP and is slightly worse than \InfTuckerEx on NELL. Note again that \InfTuckerEx uses DPM prior to enhance the factorization while our model doesn't; finally, all the nonlinear factorization methods outperform GigaTensor, a distributed CP factorization algorithm by a large margin, confirming the advantages of nonlinear factorizations on large data.  In terms of speed, our algorithm is much faster than GigaTensor and DinTucker. For example, on DBLP dataset, the average per-iteration running time were 1.45, 15.4 and 20.5 minutes for our model, GigaTensor and DinTucker, respectively. This is not surprising, because (1) our model uses the data sparsity and can exclude numerous, meaningless zero elements from training; (2) our algorithm is based on \spark, a more efficient \mapreduce system than \hadoop; (3) our algorithm gets rid of data shuffling and can fully exploit the memory-cache mechanism of \spark.
%The improvement of the speed may come from two aspects: First, our model can use data sparsity and avoid dealing with massive, meaningless zero elements; second, our algorithm is based on \spark, a more efficient distributed computing platform and fully exploits the memory-cache mechanism of \spark by avoiding the inefficient data shuffling operation, which is inevitable on \hadoop.

\subsection{Application on Click-Through-Rate Prediction}
In this section, we report the results of applying our nonlinear tensor factorization approach on Click-Through-Rate (CTR) prediction for online advertising. 

We used the online ads click log from a major Internet company, from which we extracted a four mode tensor (\textit{user}, \textit{advertisement}, \textit{publisher}, \textit{page-section}). We used the first three days's log on May 2015,  trained our model on one day's data and used it to predict the click behaviour on the next day.   
%The statistics of the data are listed in Table \ref{tb:ctr-data}. 
The sizes of the extracted tensors for the three days are $179K \times 81K \times 35 \times 355$, $167K \times 78K \times 35 \times 354$ and $213K \times 82K \times 37 \times 354$ respectively. These tensors are very sparse ($2.7\times 10^{-8}\%$ nonzeros on average). In other words, the observed clicks are very rare. However, we do not want our prediction completely bias toward zero (\ie non-click); otherwise, ads ranking and recommendation will be infeasible. Thus we sampled non-clicks of the same quantity as the clicks for training and testing. Note that training CTR prediction models with comparable clicks and non-click samples is common in online advertising systems~\citep{agarwal2014laser}. The number of training and testing entries used for the three days are $(109K,99K)$, $(91K,103K)$ and $(109K, 136K)$ respectively. 

We compared with popular methods for CTR prediction, including logistic regression and linear SVM, where each tensor entry is represented by a set of binary features according to the indices of each mode in the entry. %We used the implementation in spark MLlib. 

The results are reported in Table \ref{tb:ctr-pred}, in terms of AUC. It shows that our model improves logistic regression and linear SVM by a large margin, on average $20.7\%$ and $20.8\%$ respectively.  %Note that GigaTensor obtains worst performance. This might because GigaTensor takes into account all the tensor elements in factorization (although using data sparsity to save computation) and achieves severely biased prediction. 
Therefore, although we have not incorporated side features, such as user profiles and advertisement attributes, our tentative experiments have shown a promising potential of our model on CTR prediction task.
\begin{table}[htbp]
	%\vspace{-0.2in}
	\caption{CTR prediction accuracy on the first three days of May 2015. "1-2" means using May 1st's data for training and May 2nd's data for testing; similar are "2-3" and "3-4".} 
	\small
	\centering  
	\begin{tabular}{lccc}
		\hline \hline
		%{Method} & {5/1 pred. 5/2} & {5/2 pred. 5/3} & {5/3 pred. 5/4}
		
		Method & 1-2 & 2-3 & 3-4 \\
		\hline
		Logistic regression & 0.7360 & 0.7337 & 0.7538 \\
		Linear SVM & 0.7414 & 0.7332 & 0.7540 \\
		%GigaTensor & 0.6321 & 0.6173 & 0.5951 \\
		Our model & \textbf{0.8925} & \textbf{0.8903} & \textbf{0.9054}\\
		\hline \hline
	\end{tabular}\label{tb:ctr-pred}
\end{table}

%\vspace{-0.15in}
\section{Conclusion}
%\vspace{-0.1in}
In this paper, we have proposed a new nonlinear and flexible tensor factorization model. By disposing of the Kronecker-product covariance structure, the model can properly exploit the data sparsity and is flexible to incorporate any subset of meaningful tensor entries for training. Moreover, we have derived a tight ELBO for both continuous and binary problems, based on which we further developed an efficient distributed variational inference algorithm in \mapreduce framework. In the future, we will consider applying asynchronous inference on the tight ELBO, such as \citep{smyth2009asynchronous},  to further improve the scalability of our model.

\bibliographystyle{newapa}
\bibliography{DFTF}

%\vspace{-0.1in}
\section*{Supplementary Material}
In this extra material, we provide the details about the derivation of the tight variational evidence lower bound of our proposed GP factorization model (Section \ref{sect:tvelbo}) as well as its gradient calculation (Section \ref{sect:grad}). Moreover, we give the convergence proof of the fixed point iteration used in our distributed inference algorithm for binary tensor (Section \ref{sect:fixed}). %Finally, we show the experimental results of our model's application on real-world click-through-rate prediction problem (Section \ref{sect:ctr}). 
\setcounter{section}{0}
\section{Tight Variational Evidence Lower Bound}\label{sect:tvelbo}
The naive variational evidence lower bound (ELBO) derived from the sparse Gaussian process framework (see Section 4.1 of the main paper) is given by 
\begin{align}
L_1(\Ucal, \B, q(\v)) &= \log(p(\Ucal)) + \int q(\v) \log\frac{ p(\v|\B) }{ q(\v)} \d \v \nonumber \\
&+  \sum\nolimits_j \int q(\v) F_\v(y_{\bi_j},\beta) \d \v \label{eq:elbo_r-c}
\end{align}
for continuous tensor  
and 
\begin{align}
& L_2(\Ucal, \B, q(\v),q(\z)) = \log(p(\Ucal)) + \int q(\v) \log( \frac{p(\v|\B)}{q(\v)} )\d \v  \nonumber\\
&+ \sum\nolimits_j q(z_j)\log(\frac{p(y_{\bi_j}|z_j)}{q(z_j)}) + \sum\nolimits_j  \int q(\v) \int q(z_j)F_\v(z_j, 1) \d z_j \d \v  \label{eq:elbo_b-c} 
\end{align}
for binary tensor, where $F_\v(\cdot_j, *) =  \int \log\big(\N(\cdot_j | m_{\bi_j}, *) \big)\N(m_{\bi_j}|\mu_j, \sigma_j^2)\d m_{\bi_j}$ and $p(\v|\B) = \N(\v|\0, K_{BB})$.
Our goal is to further obtain a tight ELBO that subsumes the optimal variational posterior (\ie $q(\v)$ and $q(\z)$) so as to prevent the sequential E-M procedure for efficient parallel training and to improve the inference quality. 
\subsection{Continuous Tensor}
First, let us consider the continuous data. Given $\Ucal$ and $\B$, we use functional derivatives~\citep{bishop2006pattern} to calculate the optimal $q(\v)$. The functional derivative of $L_1$ with respect to $q(\v)$ is given by
\begin{align}
\frac{\delta L_1(q)}{\delta q(\v)} = \log\frac{p(\v|\B)}{q(\v)} - 1 + \sum\nolimits_j F_\v(y_{\bi_j},\beta). \nonumber
\end{align}
Because $q(\v)$ is a probability density function, we use Lagrange multipliers to impose the constraint and obtain the optimal $q(\v)$ by solving 
\begin{align}
\frac{\delta \big(L_1(q)+  \lambda (\int q(\v) \d \v - 1)\big)}{\delta q(\v)} &= 0, \nonumber \\
\frac{\partial \big(L_1(q)+  \lambda (\int q(\v) \d \v - 1)\big)}{\partial \lambda} &= 0. \nonumber
\end{align}
Though simple algebraic manipulations, we can obtain the optimal $q(\v)$ to be the following form
\[ 
q^*(\v) = \N(\v|\bmu,\bLambda), 
\]
where
\[
\bmu = \beta \K_{BB}(\K_{BB} + \beta\K_{BS}\K_{SB})^{-1}\K_{BS}\y, \;\;\;\;\; \bLambda=\K_{BB}(\K_{BB} + \beta\K_{BS}\K_{SB})^{-1}\K_{BB}.
\] 
 Now substituting $q(\v)$ in $L_1$ with $\N(\v|\bmu,\bLambda)$, we obtain the tight ELBO presented in \textbf{Theorem} 4.1 of the main paper:
\begin{align}
&\log\big(p(\y,\Ucal|\B)\big) \ge L_1^*(\Ucal, \B) =  \frac{1}{2}\log |\K_{BB}| - \frac{1}{2}\log|\K_{BB} + \beta\A_1|-\frac{1}{2}\beta a_2- \frac{1}{2}\beta a_3 \nonumber \\
& + \frac{\beta}{2}\tr(\K_{BB}^{-1}\A_1)-\frac{1}{2}\sum_{k=1}^K \| \U^{(k)} \|^2_F + \frac{1}{2}\beta^2 \a_4^\top(\K_{BB}+\beta\A_1)^{-1}\a_4 \nonumber \\
&+ \frac{N}{2}\log(\frac{\beta}{2\pi}), \label{eq:t_lb_r-c}
\end{align}
where $\|\cdot \|_F$ is Frobenius norm, and
\begin{align}
\A_1 &= \sum\nolimits_j k(\B, \x_{\bi_j}) k(\x_{\bi_j}, \B), \;\;\;\;\;\;\;\;\; a_2 = \sum\nolimits_j y_{\bi_j}^2, \nonumber \\
a_3 &= \sum\nolimits_j k(\x_{\bi_j}, \x_{\bi_j}),  \;\;\;\;\;\;\;\;\; \a_4 = \sum\nolimits_j k(\B, \x_{\bi_j}) y_{\bi_j} \nonumber.
\end{align}

\subsection{Binary Tensor}
Next, let us look at the binary data.  The case for binary tensors is more complex, because we have the additional variational posterior $q(\z)=\prod_j q(z_j)$. Furthermore, $q(\v)$ and $q(\z)$ are coupled in the original ELBO (see \eqref{eq:elbo_b-c}). To eliminate $q(\v)$ and $q(\z)$, we use the following steps. We first fix $q(\z)$, calculate the optimal $q(\v)$ and plug it into $L_2$ (this is similar to the continuous case) to obtain an intermediate bound, 
\begin{align}
&\hat{L}_2(q(\z), \Ucal, \B) = \max\limits_{q(\v)} L_2(q(\v), q(\z), \Ucal, \B) \nonumber \\
&=  \frac{1}{2}\log |\K_{BB}| - \frac{1}{2}\log|\K_{BB} + \A_1| - \frac{1}{2} \sum\nolimits_j \expt{z_j^2} - \frac{1}{2} a_3 + \frac{1}{2}\tr(\K_{BB}^{-1}\A_1)  \nonumber \\
&  - \frac{N}{2}\log(2\pi) + \frac{1}{2}(\K_{BS}\expt{\z})^\top(\K_{BB}+\A_1)^{-1})(\K_{BS}\expt{\z})  \nonumber \\
&   + \sum\nolimits_j \int q(z_j)\log(\frac{p(y_{\bi_j}|z_j)}{q(z_j)})\d z_j -\frac{1}{2}\sum\nolimits_{k=1}^K \| \U^{(k)} \|^2_F  \label{eq:t_lb_b_m}
\end{align}
where $\expt{\cdot}$ denotes the expectation under the variational posteriors. Note that $\hat{L}_2$ has a similar form to $L_1^*$ in \eqref{eq:t_lb_r-c}. 
%$\hat{L}_2(q(\z), \Ucal, \B) = \max\limits_{q(\v)} L_2(q(\v), q(\z), \Ucal, \B)$. It turns out that  $\hat{L}_2$ has a very similar form to $L_1$ in \eqref{eq:t_lb_r}, except that an extra summation term $\sum_j \int q(z_j)\log(\frac{p(y_{\bi_j}|z_j)}{q(z_j)})\d z_j$ is introduced  and, $y_{\bi_j}$, $y_{\bi_j}^2$ and $\beta$ are replaced by $\expt{z_{j}}$, $\expt{z_{j}^2}$ and $1$ respectively. Here $\expt{\cdot}$ denotes the expectation under the variational posteriors. 

Now we consider to calculate the optimal $q(\z)$ for $\hat{L}_2$. To this end, we calculate the functional derivative of $\hat{L}_2$ with respect to each $q(z_j)$:
\begin{align}
\frac{\delta \hat{L}_2}{\delta q(z_j)} = \log \frac{p(y_{\bi_j}|z_j)}{q(z_j)} - 1 - \frac{1}{2}z_j^2 + c_{jj}\langle z_j \rangle z_j + \sum_{t\neq j} c_{tj}\langle z_t \rangle z_j. \nonumber
\end{align}
where $c_{tj} = k(\x_{\bi_t}, \B)(\K_{BB} + \A_1)^{-1}k(\B, \x_{\bi_j})$ and $p(y_{\bi_j}|z_j) = \mathbbm{1}\big((2y_{\bi_j}-1)z_j \ge 0\big)$. 

Solving $\frac{\delta \hat{L}_2}{\delta q(z_j)}$ being $0$ with Lagrange multipliers, we find that the optimal $q(z_j)$ is a truncated Gaussian,
\begin{align}
q^*(z_j) \propto \N(z_j | c_{jj} \expt{z_j} + \sum_{t \neq j} c_{tj}  \expt{z_t},1)\mathbbm{1}\big((2y_{\bi_j}-1)z_j \ge 0\big). \nonumber
\end{align}
%Taking functional derivatives over each $q(z_j)$, we have
%$\frac{\delta \hat{L}_2(q)}{\delta q(z_j)} = -\frac{1}{2}z_j^2 + (c_{jj} \expt{z_j} + \sum_{t \neq j} c_{tj}  \expt{z_t})z_j + \log(p(y_{\bi_j}|z_j))
% - (1 + \log(q(z_j)))$
%where $c_{tj} =k(\x_{\bi_t}, \B)(\K_{BB} + \A_1)^{-1}k(\B, \x_{\bi_j})$. Solving $\frac{\delta \hat{L}_2(q)}{\delta q(z_j)} = 0$, we can see that the optimal $q(z_j)$ is a truncated Gaussian,
%$q^*(z_j) \propto \N(z_j | c_{jj} \expt{z_j} + \sum_{t \neq j} c_{tj}  \expt{z_t},1)\mathbbm{1}\big((2y_{\bi_j}-1)z_j \ge 0\big)$
This expression is unfortunately not analytical. Even if we can explicitly update each $q(z_j)$, the updating will depend on all the other variational posteriors $\{q(z_t)\}_{t\neq j}$, making distributed calculation very difficult. This arises from the quadratic term  $\frac{1}{2}(\K_{BS}\expt{\z})^\top$\\$(\K_{BB}+\A_1)^{-1}(\K_{BS}\expt{\z})$ in \eqref{eq:t_lb_b_m}, which couples all $\{\expt{z_j}\}_j$.

To resolve this issue, we introduce an extra variational parameter $\blambda$ to decouple the dependencies between $\{\expt{z_j}\}_j$ using the following lemma. 
\begin{lem}
	For any symmetric positive definite matrix $\E$, 
	\begin{align}
	\bupeta^\top \E^{-1} \bupeta \ge 2 \blambda^\top \bupeta - \blambda^\top \E \blambda. \label{ieq-1}
	\end{align} 
	The equality is achieved when $\blambda = \E^{-1} \bupeta$. 
\end{lem}
\begin{proof}
	Define the function $f(\bupeta) = \bupeta^\top \E^{-1} \bupeta$ and it is easy to see that $f(\bupeta)$ is convex because $\E^{-1} \succ 0$. Then using the convex conjugate, we have $f(\bupeta) \ge \blambda^\top \bupeta - g(\blambda)$ and $g(\blambda) \ge \bupeta^\top \blambda - f(\bupeta)$. Then by maximizing  $\bupeta^\top \blambda - f(\bupeta)$, we can obtain $g(\blambda) = \frac{1}{4}\blambda^\top \E \blambda$. Thus, $f(\bupeta)\ge \blambda^\top \bupeta - \frac{1}{4}\blambda^\top \E \blambda$. Since $\blambda$ is a free parameter, we can use $2\blambda$ to replace $\blambda$ and obtain the inequality \eqref{ieq-1}. Further, we can verify that when $\blambda = \E^{-1} \bupeta$ the equality is achieved. 
\end{proof}
We now apply the inequality on the term $\frac{1}{2}(\K_{BS}\expt{\z})^\top(\K_{BB}+\A_1)^{-1}\K_{BS}\expt{\z}$ in \eqref{eq:t_lb_b_m}. Note that the quadratic term regarding all $\{z_j\}$ now vanishes, and instead a linear term $\blambda^\top \K_{BS}\expt{\z}$ is introduced so that these annoying dependencies between $\{z_j\}_j$ are eliminated. We therefore obtain a more friendly intermediate ELBO, 
\begin{align}
&\tilde{L}_2(\Ucal,\B, q(\z), \blambda) =  \frac{1}{2}\log |\K_{BB}| - \frac{1}{2}\log|\K_{BB} + \A_1| - \frac{1}{2} \sum\nolimits_j\expt{z_j^2} - \frac{1}{2} a_3  \nonumber \\
&+ \frac{1}{2}\tr(\K_{BB}^{-1}{\A}_1) - \frac{N}{2}\log(2\pi) + \sum\nolimits_j\blambda^\top k(\B, \x_{\bi_j})\expt{z_j} - \frac{1}{2}\blambda^\top(\K_{BB}+\A_1)\blambda  \nonumber \\
& + \sum\nolimits_j \int q(z_j)\log(\frac{p(y_{\bi_j}|z_j)}{q(z_j)})\d z_j -\frac{1}{2}\sum_{k=1}^K \| \U^{(k)} \|^2_F. \label{eq:t_lb_b_m2}
\end{align}

The functional derivative with respect to $q(z_j)$ is then given by
\[
\frac{\delta \tilde{L}_2}{\delta q(z_j)} = \log \frac{p(y_{\bi_j}|z_j)}{q(z_j)} - 1 - \frac{1}{2}z_j^2 + \blambda^\top k(\B, \x_{\bi_j}) z_j.
\]
Now solving $\frac{\delta \tilde{L}_2}{\delta q(z_j)} = 0$, we see that the optimal variational posterior has an analytical form: 
\begin{align}
q^*(z_j) \propto \N(z_j | \blambda^\top k(\B,x_{\bi_j}), 1)\mathbbm{1}\big((2y_{\bi_j}-1)z_j \ge 0\big).\nonumber
\end{align}
Plugging each $q^*(z_j)$ into \eqref{eq:t_lb_b_m2}, we finally obtain the tight ELBO as presented in \textbf{Theorem} 4.2 of the main paper:
\begin{align}
&\log\big(p(\y,\Ucal|\B)\big) \ge L_2^*(\Ucal, \B, \blambda) =  \frac{1}{2}\log |\K_{BB}| - \frac{1}{2}\log|\K_{BB} + {\A}_1| - \frac{1}{2} {a}_3\nonumber \\
&+ \sum_j \log\big(\Phi((2y_{\bi_j}-1)\blambda^\top k(\B, \x_{\bi_j}))\big)- \frac{1}{2}\blambda^\top\K_{BB}\blambda  + \frac{1}{2}\tr(\K_{BB}^{-1}{\A}_1) \nonumber \\
&-\frac{1}{2}\sum_{k=1}^K \| \U^{(k)} \|^2_F. \label{eq:t_lb_b-c}
\end{align}

\section{Gradients of the Tight ELBO}\label{sect:grad}
In this section, we present how to calculate the gradients of the tight ELBOs in \eqref{eq:t_lb_r-c} and \eqref{eq:t_lb_b-c} with respect to the latent factors $\Ucal$, the inducing points $\B$ and the kernel parameters.  

Let us first consider the tight ELBO for continuous data. Because $\Ucal$, $\B$ and the kernel parameters are all inside the terms involving the kernel functions, such as $\K_{BB}$ and $\A_1$, we calculate the gradients with respect to these terms first and then use the chain rule to calculate the gradients with respect to $\Ucal$ and $\B$ and the kernel parameters. Specifically, we  consider the derivatives with respect to $\K_{BB}$, $\A_1$, $a_3$ and $\a_4$. Using matrix derivatives and algebras \citep{minka2000old}, we obtain 
\begin{align}
\d L_1^* &= \frac{1}{2}\tr\big((\K_{BB}^{-1} - (\K_{BB} + \beta\A_1)^{-1})\d\K_{BB}\big)  -\frac{\beta}{2}\tr\big((\K_{BB} + \beta\A_1)^{-1}\d \A_1\big) \nonumber \\
&- \frac{\beta}{2} \d a_3 - \frac{\beta}{2}\tr(\K_{BB}^{-1}\A_1\K_{BB}^{-1}\d\K_{BB}) + \beta^2\tr(\a_4^\top(\K_{BB} + \beta\A_1)^{-1}\d\a_4) \nonumber \\
&+ \frac{\beta}{2}\tr(\K_{BB}^{-1}\d \A_1) - \frac{1}{2}\beta^2\tr\big((\K_{BB} + \beta\A_1)^{-1}\a_4\a_4^\top(\K_{BB} + \beta\A_1)^{-1}\d \K_{BB}\big) \nonumber \\
& - \frac{1}{2}\beta^3\tr\big((\K_{BB} + \beta\A_1)^{-1}\a_4\a_4^\top(\K_{BB} + \beta\A_1)^{-1}\d\A_1\big). \label{gd1}
\end{align}

Next, we calculate the derivatives  $\d \K_{BB}$, $\d \A_1$, $\d a_3$ and $\d \a_4$, which  depend on the specific kernel function form used in the model. For example, if we use the linear kernel, $\d \K_{BB} = 2\B^\top\d\B$ and $\d \A_1 = \sum\nolimits_{j=1}^N k(\B,\x_{\bi_j})( \x_{\bi_j} \d\B^\top +\d \x_{\bi_j} \B^\top ) + (\d\B \x_{\bi_j}^\top + \B \d \x_{\bi_j}^\top)k(\x_{\bi_j},\B)$ where $\x_{\bi_j} = [\u^{(1)}_{i_{j1}}, \ldots, \u^{(K)}_{i_{jK}}]$. Note that because $\A_1$, $a_3$ and $\a_4$ all have additive structures which involve individual tensor entry $\bi_j$ ($1 \le j \le N$) and the major computation of the derivatives in \eqref{gd1} also involve similar summations, the computation of the final gradients with respect to $\Ucal$ and $\B$ and the kernel parameters can easily be performed in parallel. 
\cmt{
	Finally, the gradient of $L_1^*$ with respect to $\beta$ is given by
	\begin{align}
	\frac{\d L_1^*}{\d \beta} &= \frac{N}{2}\frac{1}{\beta} - \frac{1}{2}\tr\big((\K_{BB} + \beta\A_1)^{-1}\A_1\big) - \frac{1}{2}a_2 - \frac{1}{2}a_3 +\beta \a_4^\top (\K_{BB} + \beta\A_1)^{-1}\a_4 \nonumber \\
	&+\frac{1}{2}\tr(\K_{BB}^{-1}\A_1)- \frac{1}{2}\beta^2 \a_4^\top (\K_{BB} + \beta\A_1)^{-1}\A_1(\K_{BB} + \beta\A_1)^{-1}\a_4.
	\end{align}
}

The gradient calculation for the tight ELBOs for binary tensors is very similar to the continuous case. Specifically, we obtain
\begin{align}
\d L_2^* &= \frac{1}{2}\tr\big(\K_{BB}^{-1} - (\K_{BB} + \A_1)^{-1}\d \K_{BB}\big) -\frac{1}{2}\tr\big((\K_{BB} + \A_1)^{-1}\d \A_1\big) \nonumber \\
&-\frac{1}{2}\d a_3- \frac{1}{2}\tr(\K_{BB}^{-1}\A_1\K_{BB}^{-1}\d\K_{BB}) + \frac{1}{2}\tr(\K_{BB}^{-1}\d \A_1) - \frac{1}{2}\tr(\blambda\blambda^\top\d \K_{BB}) \nonumber \\
&+ \sum_{j=1}^N (2y_{\bi_j}-1)\frac{\N\big(\blambda^\top k(\B, \x_{\bi_j})|0,1\big)}{\Phi\big((2y_{\bi_j}-1)\blambda^\top k(\B, \x_{\bi_j})\big)} \blambda^\top \d k(\B, \x_{\bi_j}).
\end{align}
We can then calculate the derivatives  $\d \K_{BB}$, $\d \A_1$, $\d a_3$ and each $\d k(\B, \x_{\bi_j})(1 \le j \le N)$ and then apply the chain rule to calculate the gradient with respect to $\Ucal$, $\B$ and the kernel parameters.  
%Arranging terms in \eqref{gd1} and using chain rules, we have
%\begin{align}
%\frac{\partial L_1^*}{\partial \B} = \C_1 \frac{\partial \K_{BB}}{\partial \B} + \C_2 \frac{\partial \A_1}{\partial \B} + \C_3 \frac{\partial \a_4}{\partial \B}
%\end{align}

\section{Fixed Point Iteration for $\blambda$}\label{sect:fixed}
In this section, we give the convergence proof of the fixed point iteration of the variational parameters $\blambda$ in the tight ELBO for binary tensors. While $\blambda$ can be jointly optimized via gradient based approaches with $\Ucal$, $\B$ and the kernel parameters, we empirically find that combining this fixed point iteration can converge much faster. The fixed point iteration is given by 
\begin{align}
\blambda^{(t+1)} = (\K_{BB} + \A_1)^{-1} (\A_1 \blambda^{(t)} + \a_5)\label{eq:fix_point-c}
\end{align}
where  
\begin{align}
\A_1 &= \sum\nolimits_j k(\B, \x_{\bi_j}) k(\x_{\bi_j}, \B), \nonumber \\
\a_5 &= \sum_j  k(\B, \x_{\bi_j})(2y_{\bi_j}-1)\frac{\N\big(k(\B, \x_{\bi_j})^\top \blambda^{(t)}|0,1\big)}{\Phi\big((2y_{\bi_j}-1) k(\B, \x_{\bi_j})^\top \blambda^{(t)}\big)}.\nonumber
\end{align}

We now show that the fixed point iteration not only always converges, but also improves the ELBO in \eqref{eq:t_lb_b-c} after every update of $\blambda$ (see \textbf{Lemma } 4.3 in the main paper).

Specifically, given $\Ucal$ and $\B$, from Section \ref{sect:tvelbo} we have
\[
L_2^*\big(\blambda^{(t)}\big) = \max\nolimits_{q(\z)} \tilde{L}_2\big(\blambda^{(t)}, q(\z)\big) = \tilde{L}_2\big(\blambda^{(t)}, q_{\blambda^{(t)}}(\z)\big)
\]
where $q_{\blambda^{(t)}}(\z)$ is the optimal variational posterior: $q_{\blambda^{(t)}}(\z)=\prod_j q_{\blambda^{(t)}}(z_j)$ and $q_{\blambda^{(t)}}(z_j) \propto \N(z_j | k(\B, \x_{\bi_j})^\top \blambda^{(t)}, 1)\mathbbm{1}\big((2y_{\bi_j}-1)z_j \ge 0\big)$. 

Now let us fix $q_{\blambda^{(t)}}(\z)$ and derive the optimal $\blambda$ by solving $\frac{\partial \tilde{L}_2}{\partial \blambda} = 0$. We then obtain the update of $\blambda$: $\blambda^{(t+1)} = (\K_{BB} + \A_1)^{-1}\big(\sum_j k(\B, \x_{\bi_j})\expt{z_j}\big)$ where $\expt{z_j}$ is the expectation of the optimal variational posterior of $z_j$ given $\blambda^{(t)}$, \ie $q_{\blambda^{(t)}}(z_j)$. Obviously, we have 
\[
\tilde{L}_2\big(\blambda^{(t)}, q_{\blambda^{(t)}}(\z)\big) \le \tilde{L}_2\big(\blambda^{(t+1)}, q_{\blambda^{(t)}}(\z)\big). 
\]
Further, because $L_2^*(\blambda^{(t)}) = \tilde{L}_2\big(\blambda^{(t)},q_{\blambda^{(t)}}(\z)\big)$ and
\[
\tilde{L}_2\big(\blambda^{(t+1)}, q_{\blambda^{(t)}}(\z)\big) \le \tilde{L}_2\big(\blambda^{(t+1)}, q_{\blambda^{(t+1)}}(\z)\big) =  L_2^*(\blambda^{(t+1)})
\]
we conclude that $L_2^*(\blambda^{(t)}) \le L_2^*(\blambda^{(t+1)})$. Now, we plug the fact that $\expt{z_j} = w_j^{(t)} + k(\B, \x_{\bi_j})^\top \blambda^{(t)}$  given $q_{\blambda^{(t)}}(z_j)$ into the calculation of $\blambda^{(t+1)}$, merge and arrange the terms. We then obtain the fixed point iteration for $\blambda$ in \eqref{eq:fix_point-c}. Finally since $L_2^*$ is upper bounded by the log model evidence, the fixed point iteration always converges. %Note that the fixed point iteration of $\blambda$ for $L_2^*$ implicitly optimizes $\tilde{L}_2$ in an Expectation-Maximization (EM) procedure. 

\end{document}